\documentclass[10pt,twoside,compsoc]{IEEEtran}
\ifCLASSINFOpdf
\else
\fi
\usepackage[brazil, english]{babel}
\usepackage{hyperref}
\usepackage{graphicx}
\usepackage{amsmath, amssymb, amsthm, bm, algorithm, algorithmic}
\usepackage{color, url, balance, times, helvet, courier}
\usepackage{booktabs, threeparttable}
\usepackage[square,comma,sort&compress,numbers]{natbib}
\usepackage{threeparttable,multirow,adjustbox,threeparttable}
\newtheorem{lem}{Lemma}
\newtheorem{assumption}{Assumption}
\newtheorem{theorem}{Theorem}

\theoremstyle{remark}

\usepackage{changes}
\colorlet{Changes@Color}{red}
\usepackage{orcidlink}

\graphicspath{{image/}}

\begin{document}
\title{Graph Regularized NMF with $\ell_{2,0}$-norm for Unsupervised Feature Learning}
\author{Zhen Wang and Wenwen Min$^*$\textsuperscript{\orcidlink{0000-0002-2558-2911}}
\IEEEcompsocitemizethanks{
\IEEEcompsocthanksitem Zhen Wang and Wenwen Min is with the School of Information Science and Engineering, Yunnan University, Kunming 650091, Yunnan, China, China. E-mail: minwenwen@ynu.edu.cn.
}
\thanks{Manuscript received XX, 2023; revised XX, 2023.
\newline (Corresponding authors: Wenwen Min)}}

% 显示杂志的名称
%\markboth{IEEE TRANSACTIONS ON KNOWLEDGE AND DATA ENGINEERING, 2023}
%{Min \MakeLowercase{\textit{et al.}}: Graph Regularized NMF with $\ell_{2,0}$-norm}

\IEEEtitleabstractindextext{
\begin{abstract}
Nonnegative Matrix Factorization (NMF) is a widely applied technique in the fields of machine learning and data mining. 
Graph Regularized Non-negative Matrix Factorization (GNMF) is an extension of NMF that incorporates graph regularization constraints. 
GNMF has demonstrated exceptional performance in clustering and dimensionality reduction, effectively discovering inherent low-dimensional structures embedded within high-dimensional spaces. 
However, the  sensitivity of GNMF to noise limits its stability and robustness in practical applications. 
In order to enhance feature sparsity and mitigate the impact of noise while mining row sparsity patterns in the data for effective feature selection, we introduce the $\ell_{2,0}$-norm constraint as the sparsity constraints for GNMF. 
We propose an unsupervised feature learning framework based on GNMF\_$\ell_{20}$ and devise an algorithm based on PALM and its accelerated version to address this problem. 
Additionally, we establish the convergence of the proposed algorithms and validate the efficacy and superiority of our approach through experiments conducted on both simulated and real image data.
\end{abstract}
}
\maketitle
\IEEEpeerreviewmaketitle

\section{Introduction}
\IEEEPARstart{N}{onnegative} matrix factorization (NMF) is a widely used technique in the fields of machine learning \cite{hassani2021text,yu2020correntropy,jiao2020hyper, he2021boosting} and data mining \cite{gao2017local,Ding2010Convex,Wang2013Non,fu2019nonnegative,liu2017regularized, wang2022generalized}. It decomposes a nonnegative data matrix into the product of two nonnegative matrices, where one matrix represents the latent feature representation of the samples and the other matrix represents the weights of the features. NMF's advantage lies in its ability to discover latent structures and patterns in the data, making it applicable in tasks such as clustering, dimensionality reduction, and feature learning. However, traditional NMF methods do not consider the local and global structural information of the data, leading to suboptimal performance on complex datasets.

To overcome this limitation, Graph Regularized Nonnegative Matrix Factorization (GNMF) has been introduced \cite{cai2010graph,dai2019discriminant,mu2021automatic,xiu2021fault, li2020robust}. GNMF incorporates the prior information of the graph structure to assist in feature learning and effectively explore the latent clustering structure in the data. In GNMF, the similarity between samples is mapped onto a graph and combined with the nonnegative matrix factorization approach for feature learning. This allows GNMF to better preserve the local neighborhood and global consistency of the samples, thereby improving the performance of clustering and dimensionality reduction. However, traditional GNMF methods exhibit some sensitivity when dealing with noisy data, limiting their stability and robustness in practical applications, sparse constraints need to be introduced \cite{huang2018robust,zhu2017robust}.

Sparse constraints help improve the uniqueness of decomposition and enhance locally based representation, and in practice, they are almost necessary. The introduction of $\ell_2$ and $\ell_1$ norms serves to mitigate the impact of noise and outliers \cite{huang2018robust,wang2015characteristic}, and the $\ell_{2,1}$-norm can be utilized to quantify the error in matrix factorization, thereby enhancing robustness \cite{zhu2017robust, zhu2017non}. However, these sparsity constraints mentioned above do not pay attention to the row sparsity pattern in the data matrix $\boldsymbol{X}$, making these sparse methods unable to select important features in the data matrix for cluster analysis.

To enhance the sparsity of GNMF and mitigate the influence of noise, this paper introduces the $\ell_{2,0}$-norm constraint as an improvement. The $\ell_{2,0}$-norm constraint is a regularization method based on the $\ell_{2,0}$-norm, which effectively encourages the reduction of less significant features in the feature matrix, pushing them toward zero and thus achieving a sparse representation of features \cite{Towards13,pang2019efficient,du2018exploiting, li2017feature}. By incorporating the $\ell_{2,0}$-norm constraint, we can further enhance the performance of feature learning and improve the robustness against noisy data. The workflow of the GNMF\_$\ell_{20}$ model is shown in Figure \ref{fig-1}. It applies graph regularization and $\ell_{2,0}$-norm constraint on the basis of NMF, allowing the model to better mine and utilize the geometric information of the data space and extract important features. However, due to the non-convex and non-smooth nature of the $\ell_{2,0}$-norm constraint in the model, common convex optimization methods cannot solve this optimization problem.

To address this issue, this paper proposes an unsupervised feature learning framework based on GNMF\_$\ell_{20}$ and designs the corresponding PALM algorithm and its accelerated version to tackle this problem. The PALM algorithm is an alternating linear minimization method proposed for a class of non-convex and non-smooth problems that satisfy the Kurdyka-\L{ojasiewicz} (K\L) property \cite{bo2014proximal, fan2018matrix}. It optimizes the objective function by iteratively updating the characteristic matrix and weight matrix. Fortunately, our GNMF\_$\ell_{20}$ framework satisfies the K\L~property, so the PALM method can be used to solve our problem. 

In addition, there are other non-convex optimization algorithms, such as iPALM \cite{pock2016inertial} and BPL \cite{xu2017globally}. iPALM represents the inertial version of the PALM algorithm, while BPL is a non-convex optimization algorithm with global convergence based on block-coordinate updates. Both of the above methods add an extrapolation step to speed up the algorithm, inspired by these advancements, we develop a new accPALM acceleration algorithm based on the PALM method, which accelerates the convergence of the algorithm by introducing extrapolation during the iterative process and the degree of acceleration can be controlled by adjusting the extrapolation parameters. Additionally, we provide convergence analysis to guarantee the convergence and stability of the algorithm during the optimization process. We show that both the PALM and accPALM algorithms can converge to a critical point when solving the GNMF\_$\ell_{2,0}$ model.

Our main contributions are summarized as follows:
\begin{enumerate}
	\item Addressing the sensitivity of GNMF to noisy data, we introduce the $\ell_{2,0}$-norm constraint to enhance the sparsity of features, propose an unsupervised feature learning framework based on GNMF\_$\ell_{20}$.
	\item We propose a novel algorithm based on PALM and its accelerated version to solve the proposed GNMF\_$\ell_{20}$ model. Convergence analysis is provided to ensure the effectiveness and stability of the algorithm.
	\item We validate the effectiveness of the proposed method through experiments on both simulated and real-world datasets, compare it with other benchmark methods, and discuss the sparsity of features.
\end{enumerate}

\begin{figure}[h]
  \centering \includegraphics[width=1\linewidth]{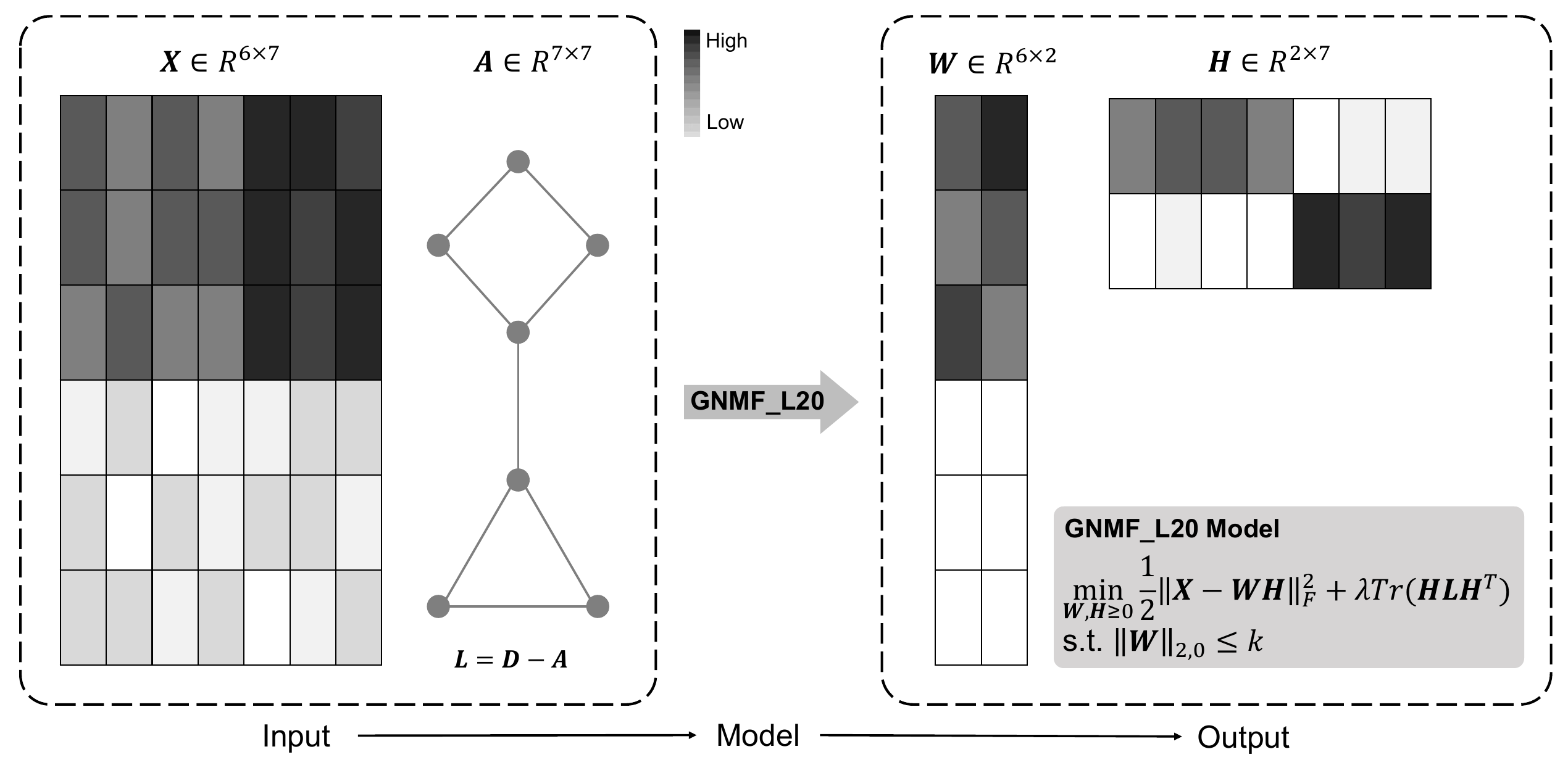}
  \caption{GNMF$\_\ell_{20}$ matrix factorization and clustering process.
  The input of GNMF$\_\ell_{20}$ is a two-dimensional matrix $\bm{X}$.
  The adjacency matrix $\bm{A}$ is constructed through the data matrix $\bm{X}$ to represent the correlation between data points.
  GNMF$\_\ell_{20}$ decomposes $\bm{X}$ into two low-dimensional matrices, namely the basis matrix $\bm{W}$ and the coefficient matrix $\bm{H}$.
  The $\bm{H}$ matrix is useful for sample clustering and visualization in low-dimensional space, while the $\hat{\bm{X}}=\bm{W}\bm{H}$ can be used for downstream analysis.
  }\label{fig-1}
\end{figure}

\section{Related work}
In this section, we present the model frameworks of NMF and GNMF along with relevant symbols and definitions.
% \subsection{Notations and definitions}
\subsection{NMF}
NMF aims to decompose a non-negative data matrix $\boldsymbol{X}$ into the product of two non-negative matrices, $\boldsymbol{W}$ and $\boldsymbol{H}$, i.e., $\boldsymbol{X}\approx\boldsymbol{W}\boldsymbol{H}$. Here, $\boldsymbol{W}$ represents the matrix of latent feature representations of the samples, and $\boldsymbol{H}$ represents the matrix of weights for the features. The loss function of NMF is based on the reconstruction error using the Euclidean distance. The standard NMF algorithm employs the Frobenius norm to quantify the dissimilarity between the original matrix $\boldsymbol{X}$ and the reconstructed matrix $\boldsymbol{WH}$. Specifically, the Frobenius norm of the matrix $\boldsymbol{X}$, denoted as $\|\boldsymbol{X}\|_{F}$, is defined as the square root of the sum of squared elements of $\boldsymbol{X}$, represented as $\|\boldsymbol{X}\|_{F}^2=\sum_{ij}\boldsymbol{X}_{ij}^2$. Given a data matrix $\boldsymbol{X}=[x_{ij}] \in \mathbb{R}_{+}^{p \times n}$ with $p$ features and $n$ samples, the NMF model \cite{lee1999learning} can be expressed as:
\begin{equation}
\begin{split}
   & \underset{\boldsymbol{W}, \boldsymbol{H}}{\operatorname{min}}\ \|\boldsymbol{X}-\boldsymbol{W} \boldsymbol{H}\|_{F}^{2} \\
   & s.t.\ \boldsymbol{W} \in \mathbb{R}_{+}^{p \times r}, \boldsymbol{H} \in \mathbb{R}_{+}^{r \times n}.
\end{split}
\end{equation}

\subsection{GNMF}
NMF attempts to find a set of basis vectors to approximate the original data. The expectation is that these basis vectors can accurately capture the intrinsic Riemannian structure inherent in the data. To observe the local geometric structure of the original data, Cai et al. proposed GNMF \cite{cai2010graph}, which incorporates a k-nearest neighbor graph $\boldsymbol{A}$, known as the adjacency matrix. For a k-nearest neighbor graph, each sample selects its k nearest samples as its neighbors based on distances or similarities, and edges are placed between these two elements.

In some cases, weights can be assigned to the elements of the adjacency matrix to indicate the strength of connections or similarities between nodes. The definition of weights can be based on distance, similarity measures, or other domain knowledge. Three of the most frequently employed are as follows \cite{chen2023graph, cai2010graph}:
\begin{enumerate}
	\item \textbf{0-1 weight:} If two vertices, $\boldsymbol{x_j}$ and $\boldsymbol{x_l}$, are interconnected, their weight $a_{jl}$ is assigned a value of 1. Conversely, if they are not connected, $a_{jl}$ is set to 0.
    \begin{equation}
    a_{j l}=\left\{\begin{array}{cl}1  &\text{if }\boldsymbol{x}_{j} \text { and } \boldsymbol{x}_{l} \text { are connected},\\
    0  &\text{otherwise}.\end{array}\right.
    \end{equation}
	\item \textbf{Gaussian kernel weight:} The Gaussian kernel function is a distance-based function, assigning higher weights to data points that are closer in distance and lower weights to data points that are farther away.
    \begin{equation}
    a_{j l}=\left\{\begin{array}{cl}e^{-\frac{\left\|x_{l}-x_{j}\right\|^{2}}{2\sigma^{2}}} &\text{if }\boldsymbol{x}_{j} \text { and } \boldsymbol{x}_{l} \text { are connected}, \\
    0 &\text {otherwise}.
    \end{array}\right.
    \end{equation}
	\item \textbf{Dot-Product Weighting:} Dot product weighting entails evaluating the dot product of two vectors to gauge their similarity or degree of correlation. This similarity measure is then harnessed to allocate weights to the vectors. It's important to underscore that in scenarios where vector normalization is implemented, the dot product aligns with the cosine similarity between the vectors.
     \begin{equation}
    a_{j l}=\left\{\begin{array}{cl}x_{j}^{T}x_{l} &\text{if }\boldsymbol{x}_{j} \text{ and } \boldsymbol{x}_{l} \text{ are connected}, \\
    0 &\text{otherwise}.\end{array}\right.
    \end{equation}
\end{enumerate}
Unless otherwise specified, Gaussian kernel weighting is adopted in this study as the method for assessing similarities.

In the study of manifold learning, given a adjacency matrix $\boldsymbol{A} \in \mathbb{R}^{n \times n}$, the smoothness of the original data in the low-dimensional representation is expressed as:
\begin{equation}
    \mathcal{R} = \sum_{j, l=1}^{n}\left\|\boldsymbol{h}_{j}-\boldsymbol{h}_{l}\right\|^{2} \boldsymbol{a}_{j l}, \label{eq:5}
\end{equation}
where $\boldsymbol{a}_{j l}$ represents the number in the $j\text{-th}$ row and $l\text{-th}$ column of the adjacency matrix $\boldsymbol{A}$, denote $h_j$ as the $j$-th column of $\boldsymbol{H}$, $h_j =[h_{1j},...,h_{rj}]^T$.
The Eq.(\ref{eq:5}) can be simplified as:
\begin{equation}
\begin{split}
    \mathcal{R} &=\sum_{j, l=1}^{n}\left\|\boldsymbol{h}_{j}-\boldsymbol{h}_{l}\right\|^{2} \boldsymbol{a}_{j l} \\
    &=\sum_{j=1}^{n} \boldsymbol{h}_{j}^{T} \boldsymbol{h}_{j} \boldsymbol{d}_{j j}-\sum_{j, l=1}^{n} \boldsymbol{h}_{j}^{T} \boldsymbol{h}_{l} \boldsymbol{a}_{j l} \\
    &=\operatorname{Tr}(\boldsymbol{H} \boldsymbol{D} \boldsymbol{H}^{T})-\operatorname{Tr}(\boldsymbol{H} \boldsymbol{A} \boldsymbol{H}^{T}) \\
    &=\operatorname{Tr}(\boldsymbol{H} \boldsymbol{L} \boldsymbol{H}^{T}), \label{eq:6}
\end{split}
\end{equation}
where the $\operatorname{Tr}(\cdot)$ represents the trace of a matrix, and $\boldsymbol{D}$ is a diagonal matrix with entries representing the column sums of $\boldsymbol{A}$ (or row sums, since $\boldsymbol{A}$ is symmetric), denoted as $\boldsymbol{d}_{jj}=\sum_{l}\boldsymbol{a}_{jl}$. The matrix $\boldsymbol{L}=\boldsymbol{D}-\boldsymbol{A}$ is referred to as the graph Laplacian, which captures the graph structure of the data.

By incorporating the graph regularization penalty term into the NMF framework, the problem of GNMF can be formulated as follows:
\begin{equation}
\begin{split}
   & \underset{\boldsymbol{W}, \boldsymbol{H}}{\operatorname{min}}\ \|\boldsymbol{X}-\boldsymbol{W} \boldsymbol{H}\|_{F}^{2}+\lambda\operatorname{Tr}(\boldsymbol{HL}\boldsymbol{H}^T) \\
   & s.t.\ \boldsymbol{W} \in \mathbb{R}_{+}^{p \times r}, \boldsymbol{H} \in \mathbb{R}_{+}^{r \times n}.
\end{split}
\end{equation}
GNMF introduces graph regularization constraint on the basis of NMF, which enhances the performance of feature learning by mapping the similarities between samples onto the graph.

% 这里要不要删掉呀，感觉有点废话了呜呜
In summary, the GNMF framework combines the capability of feature learning with the graph regularization constraint, allowing for better discovery of the underlying structure and patterns in the data. However, traditional GNMF methods are sensitive to noisy data. To enhance sparsity and mitigate the impact of noise, we introduce the $\ell_{2,0}$-norm constraint, which will be discussed in detail in the following sections.

\section{Proposed framework}
In this section, a detailed introduction is provided for the formulation of the GNMF\_$\ell_{20}$ objective function, the PALM algorithm and its accelerated version, as well as the convergence analysis.
\subsection{GNMF$\_\ell_{20}$}
First, we define a function $I(x)$ as follows:
\begin{equation*}
    I(x) = \left\{\begin{array}{rl}
    1 & \text{if } x \not= 0,\\
    0 & \text{if } x = 0.
    \end{array} \right.
\end{equation*}
For a given matrix $\boldsymbol{W}\in \mathbb{R}^{p \times r}$, the $\ell_{2,0}$-norm represents the number of non-zero rows. Specifically, the $\ell_{2,0}$-norm of $\boldsymbol{W}$ is described as follows:
\begin{equation}
    \|\boldsymbol{W}\|_{2,0}=\sum_{i=1}^{p}I(\|w^{i}\|),
\end{equation}
where $w^{i}$ denote the $i\text{-th}$ row of $\boldsymbol{W}$, $\| \cdot \|$ is the $l_{2}$-norm and $\|x\| = \sum x_i^2$.

To integrate feature selection and graph regularization constraints in the NMF model, and enhance the sparsity and stability of the model, we introduce a row-sparse GNMF model with $\ell_{2,0}$-norm constraint (GNMF\_$\ell_{20}$):
\begin{equation}\label{eq:9}
\begin{split}
    & \underset{\boldsymbol{W}, \boldsymbol{H}}{\operatorname{min}} \ \frac{1}{2}\|\boldsymbol{X}-\boldsymbol{W} \boldsymbol{H}\|_{F}^{2}+\lambda \operatorname{Tr}(\boldsymbol{H} \boldsymbol{L} \boldsymbol{H}^{T})\\
    & s.t. \boldsymbol{W} \in \mathbb{R}_{+}^{p \times r},\|\boldsymbol{W}\|_{2,0} \leq k, \boldsymbol{H} \in \mathbb{R}_{+}^{r \times n}.
\end{split}
\end{equation}
By imposing the constraint $\|\boldsymbol{W}\|_{2,0} \leq k$ , we encourage sparsity in the rows of $\boldsymbol{W}$ and select the most important $k$ features.

\subsection{Proposed algorithm based on PALM}
PALM \cite{bo2014proximal} (Proximal Alternating Linearized Minimization) algorithm is a widely used optimization algorithm for solving a class of non-convex and non-smooth minimization problems.
The key idea of the PALM algorithm is to decompose the original problem into two subproblems and iteratively solve them alternatively.

% The PALM algorithm iteratively updates two variables, typically denoted as x and y, until a stopping criterion is met. In each iteration, the algorithm first updates x by solving a linearized subproblem that approximates the original objective function.  This update involves gradient computations and a proximal operator to handle regularization terms. Then, the algorithm updates y by solving another linearized subproblem, again incorporating gradient information and the proximal operator. 
In order to apply the PALM algorithm to solve our model, we reformulate the objective function of GNMF\_$\ell_{20}$ \eqref{eq:9} as follows:
\begin{equation}
    \underset{\boldsymbol{W}, \boldsymbol{H}}{\operatorname{min}} \ F(\boldsymbol{W},\boldsymbol{H})+ \delta_{\boldsymbol{W}\in \{\mathbb{R}_{+}^{p \times r},\|\boldsymbol{W}\|_{2,0} \leq k\}} + \delta_{\boldsymbol{H}\in \{\mathbb{R}_{+}^{r \times n}\}}, \label{eq:10}
\end{equation}
% ,and $\delta_{\boldsymbol{H}\in \{\mathbb{R}_{+}^{r \times n}\}}$ is zero if $\boldsymbol{H}\in \{\mathbb{R}_{+}^{r \times n}\}$, otherwise $+\infty$
where $F(\boldsymbol{W},\boldsymbol{H})= \frac{1}{2}\|\boldsymbol{X}-\boldsymbol{W} \boldsymbol{H}\|_{F}^{2}+\lambda \operatorname{Tr}(\boldsymbol{H} \boldsymbol{L} \boldsymbol{H}^{T})$ , and the indicator function $\delta_{\boldsymbol{H}\in \{\mathbb{R}_{+}^{r \times n}\}}$ is defined as follows:
\begin{equation*}
    \delta_{\boldsymbol{H}\in \{\mathbb{R}_{+}^{r \times n}\}} = \left\{\begin{array}{cl}
    0  &\text{if } \boldsymbol{H}\in \{\mathbb{R}_{+}^{r \times n}\},\\
    +\infty  &\text{otherwise},
    \end{array} \right.
\end{equation*}
the indicator function $\delta_{\boldsymbol{W}\in \{\mathbb{R}_{+}^{p \times r},\|\boldsymbol{W}\|_{2,0} \leq k\}}$ has a similar definition. Furthermore, let:
\begin{equation*}
    f(\boldsymbol{W})=\delta_{\boldsymbol{W}\in \{\mathbb{R}_{+}^{p \times r},\|\boldsymbol{W}\|_{2,0} \leq k\}},
\end{equation*}
\begin{equation*}
    g(\boldsymbol{H})=\delta_{\boldsymbol{H}\in \{\mathbb{R}_{+}^{r \times n}\}}.
\end{equation*}
The objective function of GNMF\_$\ell_{20}$ \eqref{eq:9} can be transformed as follows:
\begin{equation}
    \underset{\boldsymbol{W}, \boldsymbol{H}}{\operatorname{min}} \ J(\boldsymbol{W},\boldsymbol{H}): =F(\boldsymbol{W},\boldsymbol{H})+f(\boldsymbol{W})+g(\boldsymbol{H}).
    \label{eq:11}
\end{equation}
Before giving the formal algorithm framework, we need to demonstrate that the PALM algorithm can be applied to our GNMF\_$\ell_{20}$ model.
Theorem 1 in \cite{min2022structured} shows that $F(\boldsymbol{W},\boldsymbol{H})$, $f(\boldsymbol{W})$ and $g(\boldsymbol{H})$ are semi-algebraic functions, therefore, $J(\boldsymbol{W},\boldsymbol{H})$ is also a semi-algebraic function. Furthermore, it is obvious that $J(\boldsymbol{W},\boldsymbol{H})$ is a proper and lower semicontinuous function. Combined with Theorem 3 of \cite{bo2014proximal}, we can know that $J(\boldsymbol{W},\boldsymbol{H})$ satisfies the K\L~property, which means that the PALM algorithm can be used to solve our GNMF\_$\ell_{20}$ model.

We first define the proximal map here, let $\sigma:\mathbb{R}^{d}\to (-\infty,\infty] $ be a proper and lower semicontinuous function. Given $x\in\mathbb{R}^{d}$ and $t>0$, the proximal map associated to $\sigma$ is defined by:
\begin{equation}
    \operatorname{prox}_{t}^{\sigma}(x):=\operatorname{argmin}\left\{\sigma(u)+\frac{t}{2}\|u-x\|^{2}: u \in \mathbb{R}^{d}\right\}.
    \label{eq:12}
\end{equation}
Further, when $\sigma$ is the indicator function of a nonempty and closed set $X$ as $\mathbb{R}_{+}^{r \times n}$, the proximal map \eqref{eq:12} reduces to the projection operator onto $X$, defined by:
\begin{equation}
    P_{X}(v):=\operatorname{argmin}\left\{\|u-v\|: u \in X\right\}.
\end{equation}

A fundamental algorithmic framework for solving Eq. (\ref{eq:11}) is presented in Algorithm \ref{alg-1}.
\begin{algorithm}[htbp]
    \caption{PALM algorithm for solving Eq. \eqref{eq:11}} \label{alg-1}
    \begin{algorithmic}[1] %每行显示行号，1表示每1行进行显示
        \REQUIRE $\boldsymbol{X}\in \mathbb{R}^{p\times n}$.
        \ENSURE $\boldsymbol{W}\in \mathbb{R}^{p\times r} \text{ and }\boldsymbol{H}\in \mathbb{R}^{r\times n}$.
        \STATE Initialization:$(\boldsymbol{W}^{0}, \boldsymbol{H}^{0}), k=0$.
        \STATE \textbf{While} until convergence or stopping criterion \textbf{do}
        \STATE \quad Compute $c_k$
        \STATE \quad $\boldsymbol{W}^{k+1} \in \operatorname{prox}_{c_k}^{f}\left(\boldsymbol{W}^{k}-\frac{1}{c_k}\nabla_{W} F(\boldsymbol{W}^{k},\boldsymbol{H}^{k})\right)$
        \STATE \quad Compute $d_k$
        \STATE \quad $\boldsymbol{H}^{k+1} \in \operatorname{prox}_{d_k}^{g}\left(\boldsymbol{H}^{k}-\frac{1}{d_k}\nabla_{H} F(\boldsymbol{W}^{k+1},\boldsymbol{H}^{k})\right)$
        \STATE \quad $ k \leftarrow k+1$
        \RETURN $\boldsymbol{W}^{k}$ and $\boldsymbol{H}^{k}$
    \end{algorithmic}
\end{algorithm}
Based on Algorithm \ref{alg-1}, we also develop an accelerated proximal gradient method with adaptive momentum in Algorithm \ref{alg-2} to accelerate the convergence of the objective function. This algorithm combines the proximal gradient method and the extrapolation method to ensure that the objective function is monotonic and non-increasing. In each step, we compare the function value of J after the projected gradient descent of the extrapolated point and the original point respectively, and select a method with a smaller function value for the next iteration, and the extrapolation parameter $\beta$ change dynamically as the iteration proceeds.
% In addition to this, we also consider an adaptive method for selecting extrapolation parameters that can significantly improve the algorithm.

\begin{algorithm}[htbp]
    \caption{accPALM algorithm with adaptive momentum} \label{alg-2}
    \begin{algorithmic}[1] %每行显示行号，1表示每1行进行显示
        \REQUIRE $\boldsymbol{X}\in \mathbb{R}^{p\times n}$.
        \ENSURE $\boldsymbol{W}\in \mathbb{R}^{p\times r} \text{ and }\boldsymbol{H}\in \mathbb{R}^{r\times n}$.
        \STATE Initialization:$(\boldsymbol{W}^{-1}, \boldsymbol{H}^{-1})=(\boldsymbol{W}^{0}, \boldsymbol{H}^{0}),\beta_{0}\in\left[0, \beta_{\max }\right],\beta_{\max }<1 ,k=0, t>1$.
        \REPEAT
        \STATE $\widetilde{\boldsymbol{W}}^{k} = \boldsymbol{W}^{k}+\beta_{k}(\boldsymbol{W}^{k}-\boldsymbol{W}^{k-1})$
        
        \STATE compute $c_k$
        \STATE $\widetilde{\boldsymbol{W}}^{k+1} \in \operatorname{prox}_{c_k}^{f}\left(\widetilde{\boldsymbol{W}}^{k}-\frac{1}{c_k}\nabla_{\boldsymbol{W}}F(\widetilde{\boldsymbol{W}}^{k},\boldsymbol{H}^{k})\right)$
        \STATE $\widetilde{\boldsymbol{H}}^{k} = \boldsymbol{H}^{k}+\beta_{k}(\boldsymbol{H}^{k}-\boldsymbol{H}^{k-1})$
        \STATE compute $d_k$
        \STATE $ \widetilde{\boldsymbol{H}}^{k+1} \in \operatorname{prox}_{d_k}^{g}\left(\widetilde{\boldsymbol{H}}^{k}-\frac{1}{d_k}\nabla_{\boldsymbol{H}} F(\widetilde{\boldsymbol{W}}^{k+1},\widetilde{\boldsymbol{H}}^{k})\right)$
        \IF{$F(\widetilde{\boldsymbol{W}}^{k+1}, \widetilde{\boldsymbol{H}}^{k+1}) \leq F\left(\boldsymbol{W}^{k}, \boldsymbol{H}^{k}\right)-\rho_{0}\Vert(\bm{W}^{k+1}-\widetilde{\bm{W}}^{k},\bm{H}^{k+1}-\widetilde{\bm{H}}^{k}) \Vert^2$}
        \STATE Update $\beta_{k+1} = \min(t*\beta_{k},~\beta_{\max})$
        \STATE $\boldsymbol{W}^{k+1}=\widetilde{\boldsymbol{W}}^{k+1}, \boldsymbol{H}^{k+1}=\widetilde{\boldsymbol{H}}^{k+1}$
        \ELSE
        
        \STATE $\boldsymbol{W}^{k+1} \in \operatorname{prox}_{c_k}^{g}\left(\boldsymbol{W}^{k}-\frac{1}{d_k}\nabla_{\boldsymbol{W}} F(\boldsymbol{W}^{k},\boldsymbol{H}^{k})\right)$
        \STATE $\boldsymbol{H}^{k+1} \in \operatorname{prox}_{c_k}^{g}\left(\boldsymbol{H}^{k}-\frac{1}{c_k}\nabla_{\boldsymbol{H}} F(\boldsymbol{W}^{k+1},\boldsymbol{H}^{k})\right)$
        \STATE Update $\beta_{k+1}=\beta_{k}/t$
        \ENDIF
        \STATE \text{set }$k = k+1$
        \UNTIL $\frac{\Vert(\boldsymbol{W}_k,\boldsymbol{H}_k)-(\boldsymbol{W}_{k-1},\boldsymbol{H}_{k-1})\Vert}{\Vert(\boldsymbol{W}_{k-1},\boldsymbol{H}_{k-1})\Vert}<\epsilon$.
        \RETURN $\boldsymbol{W}:=\boldsymbol{W}_k$ and $\boldsymbol{H}:=\boldsymbol{H}_k$
    \end{algorithmic}
\end{algorithm}

Next, we will compute the partial derivatives of $F$ with respect to $\boldsymbol{W}$ and $\boldsymbol{H}$:
%每个子行都有唯一编号
\begin{subequations}\label{eqn-10}
  \begin{align}
    \nabla_{H} F =&\boldsymbol{W}^{T} \boldsymbol{W} \boldsymbol{H}-\boldsymbol{W}^{T} \boldsymbol{X}+2 \lambda \boldsymbol{H} \boldsymbol{L}, \\
    \nabla_{W} &F =\boldsymbol{W} \boldsymbol{H} \boldsymbol{H}^{T}-\boldsymbol{X} \boldsymbol{H}^{T}.
  \end{align}
\end{subequations}
The Hessian matrices of F on $\bm{W}$ and $\bm{H}$ are as follows:
\begin{subequations}
\begin{align}
    \nabla_{H}^{2} F=&\boldsymbol{I}_{n} \otimes(\boldsymbol{W}^{T} \boldsymbol{W})+(2 \lambda \boldsymbol{L}^{T}) \otimes \boldsymbol{I}_{r}, \\
    &\nabla_{W}^{2} F=(\boldsymbol{H} \boldsymbol{H}^{T}) \otimes \boldsymbol{I}_{p},
\end{align}
\end{subequations}
where $I_p \in \mathbb{R}^{p \times p}$ is an identity matrix,and $\otimes$ is the Kronecker product.
We have known that $\nabla_W F$ and $\nabla_H F$ are Lipschitz continuous \cite{guan2012nenmf}, so the Lipschitz constant of $\nabla_{W} F$ is the largest singular value of $\nabla_{W}^{2} F$, $i.e.$, $L_{W}=\|\boldsymbol{H}\boldsymbol{H}^{T}\|_{2}$. Similarly, we can obtain that $L_{H}=\|\boldsymbol{W}^{T}\boldsymbol{W}^{T}\|_{2}+\|2\lambda\boldsymbol{L}^T\|_{2}$. The step size $c_k$ and $d_k$ in Algorithms \ref{alg-1} and \ref{alg-2} need to meet $c_k\geq L_W$ and $d_k\geq L_H$, and we can simply set $c_{k}=L_{W}$ and $d_{k}=L_{H}$.

Since both non-coupling terms in the equation are indicator functions, the proximal mapping operations in the iterative steps of $\boldsymbol{W}$ and $\boldsymbol{H}$ can be simplified to the following proximal projections:
\begin{equation}\label{eq:16}
    P_{f}(\overline{\boldsymbol{W}}):=\operatorname{argmin}\left\{\|\boldsymbol{W}-\overline{\boldsymbol{W}}\|: \boldsymbol{W} \in \{\mathbb{R}_{+}^{p \times r},\|\boldsymbol{W}\|_{2,0} \leq k\}\right\},
\end{equation}
\begin{equation}\label{eq:17}
    P_{g}(\overline{\boldsymbol{H}}):=\operatorname{argmin}\left\{\|\boldsymbol{H}-\overline{\boldsymbol{H}}\|: \boldsymbol{H} \in \{\mathbb{R}_{+}^{r \times n}\}\right\},
\end{equation}
where $\overline{\boldsymbol{W}}:=\boldsymbol{W}^{k}-\frac{1}{c_k}\nabla_{W} F(\boldsymbol{W}^{k},\boldsymbol{H}^{k})$,and $\overline{\boldsymbol{H}}:=\boldsymbol{H}^{k}-\frac{1}{d_k}\nabla_{H} F(\boldsymbol{W}^{k},\boldsymbol{H}^{k})$.

According to Proposition 2 and Proposition 3 in references \cite{min2022structured}, we can know the closed solution of Eq. (\ref{eq:16}) and (\ref{eq:17}) can be obtained as follows:
\begin{equation}
    P_{f}(\overline{\boldsymbol{W}}):=RS_{k}(P_{+}(\overline{\boldsymbol{W}})),
\end{equation}
\begin{equation}
    P_{g}(\overline{\boldsymbol{H}}):=P_{+}(\overline{\boldsymbol{H}}).
\end{equation}
Then we can give the algorithm \ref{alg-3} for solving Eq. (\ref{eq:9}):

\begin{algorithm}
    \caption{accPALM algorithm for solving Eq. (\ref{eq:9})} \label{alg-3}%标题
    \begin{algorithmic}[1] %每行显示行号，1表示每1行进行显示
        \REQUIRE $\boldsymbol{X}\in \mathbb{R}^{p\times n}$.
        \ENSURE $\boldsymbol{W}\in \mathbb{R}^{p\times r} \text{ and }\boldsymbol{H}\in \mathbb{R}^{r\times n}$.
        \STATE Initialization:$(\boldsymbol{W}^{-1}, \boldsymbol{H}^{-1})=(\boldsymbol{W}^{0}, \boldsymbol{H}^{0}),\beta_{0}\in\left[0, \beta_{\max }\right],\beta_{\max }<1 ,k=0$.
        \REPEAT
        \STATE $\widetilde{\boldsymbol{W}}^{k} = \boldsymbol{W}^{k}+\beta_{k}(\boldsymbol{W}^{k}-\boldsymbol{W}^{k-1})$
        \STATE $c_k=\|\boldsymbol{H}\boldsymbol{H}^{T}\|_{2}$
        \STATE $\boldsymbol{W}^{k+1} = RS_{k}\left(P_{+}\left(\widetilde{\boldsymbol{W}}^{k}-\frac{1}{c_k}\nabla_{\boldsymbol{W}}F(\widetilde{\boldsymbol{W}}^{k},\boldsymbol{H}^{k})\right)\right)$
        \STATE $\widetilde{\boldsymbol{H}}^{k} = \boldsymbol{H}^{k}+\beta_{k}(\boldsymbol{H}^{k}-\boldsymbol{H}^{k-1})$
        \STATE $d_k=\|\boldsymbol{W}^{T}\boldsymbol{W}^{T}\|_{2}+\|2\lambda\boldsymbol{L}^T\|_{2}$
        \STATE $\boldsymbol{H}^{k+1} = P_{+}\left(\widetilde{\boldsymbol{H}}^{k}-\frac{1}{d_k}\nabla_{\boldsymbol{H}} F(\boldsymbol{W}^{k+1},\widetilde{\boldsymbol{H}}^{k})\right)$
        \IF {$F(\boldsymbol{W}^{k+1}, \boldsymbol{H}^{k+1}) \leq F\left(\boldsymbol{W}^{k}, \boldsymbol{H}^{k}\right)-\rho_{0}\Vert(\bm{W}^{k+1}-\widetilde{\bm{W}}^{k},\bm{H}^{k+1}-\widetilde{\bm{H}}^{k}) \Vert^2$}
        % \STATE $\boldsymbol{W}^{k+1}=\boldsymbol{W}^{k+1}, \boldsymbol{H}^{k+1}=\widetilde{\boldsymbol{H}}^{k+1}$
        \STATE Compute $\beta_{k+1}$
        \ELSE
        \STATE $\widetilde{\boldsymbol{H}}^{k} = \boldsymbol{H}^{k}, \widetilde{\boldsymbol{W}}^{k} = \boldsymbol{W}^{k}$
        \STATE repeat 5 and 8, compute $\beta_{k+1}$
        % \STATE $\boldsymbol{W}^{k+1} = RS_{k}\left(P_{+} \left(\boldsymbol{W}^{k}-\frac{1}{d_k}\nabla_{\boldsymbol{W}} F(\boldsymbol{W}^{k},\boldsymbol{H}^{k})\right)\right)$
        % \STATE $\boldsymbol{H}^{k+1} = P_{+} \left(\boldsymbol{H}^{k}-\frac{1}{c_k}\nabla_{\boldsymbol{H}} F(\boldsymbol{W}^{k+1},\boldsymbol{H}^{k})\right)$
        \ENDIF
        % \STATE Compute $\beta_{k+1}$\text{ and set }$k = k+1$
        \STATE Set $k = k+1$
        \UNTIL $\frac{\Vert(\boldsymbol{W}^k,\boldsymbol{H}^k)-(\boldsymbol{W}^{k-1},\boldsymbol{H}^{k-1})\Vert}{\Vert(\boldsymbol{W}^{k-1},\boldsymbol{H}^{k-1})\Vert}<\epsilon$.
        \RETURN $\boldsymbol{W}:=\boldsymbol{W}^k$ and $\boldsymbol{H}:=\boldsymbol{H}^k$
    \end{algorithmic}
\end{algorithm}
% 在算法3的第10行和第13行中，我们没有显式地给出外推参数\beta的更新方式，我们可以将\beta固定，也可以像算法2一样在迭代过程中自适应动态改变参数\beta，后文中通过实验我们可以得出结论自适应更新参数\btea能够得到更好的结果。如果没有特殊说明，在后文中我们均采用动态更新\beta的参数迭代方案。
In steps 10 and 13 of Algorithm \ref{alg-3}, the update method for the extrapolation parameter $\beta$ is not explicitly provided. We can choose to fix $\beta$, or dynamically adapt the parameters during the iteration process, similar to Algorithm 2. Through experiments detailed in the subsequent sections, we observe that adaptively updating the parameter $\beta$ yields improved results. Unless otherwise specified, we will employ the dynamic update scheme for $\beta$ in the following sections.

\subsection{Convergence analysis}
The convergence proof for the PALM algorithm has been presented in \cite{bo2014proximal}. Our GNMF\_$\ell_{20}$ framework adheres to the basic assumptions of PALM, ensuring that the convergence properties remain consistent with those established in \cite{bo2014proximal}. Furthermore, we provide a convergence proof for the accPALM Algorithm (\ref{alg-3}), the main ideas of the proof are similar to \cite{yang2023accelerated}. 
Before proceeding with the formal proof of Algorithm \ref{alg-3}, we introduce several foundational assumptions \cite{yang2022some}.

\begin{assumption}\label{assumption-1}
    (i) $f: \mathbb{R}^{p\times r} \rightarrow(-\infty,+\infty]$ and $g: \mathbb{R}^{r\times n} \rightarrow(-\infty,+\infty]$ are proper and lower semicontinuous functions.\\
    (ii) $\inf_{\mathbb{R}^{p\times r} \times\mathbb{R}^{r\times n}}J>-\infty$, $\inf_{\mathbb{R}^{p\times r}} f>-\infty \text{ and } \inf _{\mathbb{R}^{r\times n}}g>-\infty$.\\
    (iii) $F:\mathbb{R}^{p\times r}\times \mathbb{R}^{r\times n}\rightarrow\mathbb{R}$ is a $C^1$ function. The partial gradient $\nabla_{W} F\left(W, H\right)$ is globally Lipschitz with moduli $L_W$, and for any $W_{1}, W_{2} \in \mathbb{R}^{p \times r}$, we have
    \begin{equation}
        \left\|\nabla_{W} F\left(W_{1}, H\right)-\nabla_{W} F\left(W_{2}, H\right)\right\| \leq L_W\left\|W_{1}-W_{2}\right\|.
    \end{equation}
    Likewise, we have a similar conclusion for $\nabla_{H} F\left(W, H\right)$ and $L_H$.
\end{assumption}

\begin{lem}\label{lem-1}
    \text{(Convergence properties)} Suppose that Assumption \ref{assumption-1} holds. Let $\{(\boldsymbol{W}^k,\boldsymbol{H}^k)\}$ be a sequence generated by Algorithm \ref{alg-3}. The following assertions hold.\\
    (i) The sequences $J{(\boldsymbol{W}^k,\boldsymbol{H}^k)}$ is monotonically nonincreasing and in particular
    \begin{equation}\label{eq:21}
    \begin{split}
        &J(W^{k+1},H^{k+1}) \\
        &\leq J(\widetilde{W}^{k},\widetilde{H}^{k})-\rho_{0}\Vert(W^{k+1}-\widetilde{W}^{k},H^{k+1}-\widetilde{H}^{k}) \Vert^2,
    \end{split}
    \end{equation}
    where $\rho_{0}=\min{\{\frac{1}{2}(c_{k}-L_W),\frac{1}{2}(d_{k}-L_H)}\}$.\\
    (ii) We have
    \begin{equation}
        \lim_{k\to\infty} \Vert W^{k+1}-\widetilde{W}^{k} \Vert =0, 
        \lim_{k\to\infty} \Vert H^{k+1}-\widetilde{H}^{k} \Vert =0.
    \end{equation}
\end{lem}
\begin{proof}
    (i) According to Lemma 2 of reference \cite{bo2014proximal} and the iterative format of $W$ and $H$, we can draw the following conclusions:
    \begin{equation}
        \begin{split}
            &F(W^{k+1},\widetilde{H}^{k})+f(W^{k+1})\\
            &\le F(\widetilde{W}^{k},\widetilde{H}^{k})+f(\widetilde{W}^{k})-\frac{1}{2}(c_{k}-L_W)\Vert W^{k+1}-\widetilde{W}^{k} \Vert^{2},
        \end{split}
    \end{equation}
    \begin{equation}
        \begin{split}
           &F(W^{k+1},H^{k+1})+g(H^{k+1})\\
           &\leq F(W^{k+1},\widetilde{H}^{k})+g(\widetilde{H}^{k})-\frac{1}{2}(d_{k}-L_H)\Vert H^{k+1}-\widetilde{H}^{k} \Vert^{2}.
        \end{split}
    \end{equation}
    Adding the above two inequalities, we get
    \begin{equation}\label{eq:25}
    \begin{split}
        &J(W^{k+1},H^{k+1})\\
        &\leq J(\widetilde{W}^{k},\widetilde{H}^{k})-\rho_{0}\Vert(W^{k+1}-\widetilde{W}^{k},H^{k+1}-\widetilde{H}^{k}) \Vert^2.
    \end{split}
    \end{equation}
From Algorithm \ref{alg-3}, we can observe that the updates of $W^{k+1}$ and $H^{k+1}$ are derived from either steps 5 and 8 or step 13 in Algorithm \ref{alg-3}. Let us consider two cases. If the updates of $W^{k+1}$ and $H^{k+1}$ are obtained from steps 5 and 8, the conditional statement (step 9) implies that the objective function decreases in this iteration. If the updates of $W^{k+1}$ and $H^{k+1}$ are obtained from step 13, then Eq. (\ref{eq:25}) will degenerate into 
\begin{equation}\label{eq:26}
    \begin{split}
        &J(W^{k+1},H^{k+1})\\
        &\leq J(W^{k},H^{k})-\rho_{0}\Vert(W^{k+1}-\widetilde{W}^{k},H^{k+1}-\widetilde{H}^{k}) \Vert^2.
    \end{split}
\end{equation}
It is evident that the objective function also remains non-increasing in this iteration. Thus, we can conclude that the objective function value is non-increasing as the iterations progress. And because the objective function is lower bounded (see Assumption \ref{assumption-1}), the sequence $J(W^{k},H^{k})$ is convergent.\\
(ii)According to the conclusion of Lemma \ref{lem-1}(i) and the conditional statement (step 9), we know that no matter which step $W^{k+1}$ and $H^{k+1}$ are obtained from, Eq. (\ref{eq:26}) holds, that is
% \begin{equation}
%     J(W^{k},H^{k})-J(W^{k+1},H^{k+1}) > 0.
% \end{equation}
% There exists a $\rho_1>0$ such that the following inequality holds
\begin{equation}\label{eq:28}
    \begin{split}
        &\rho_{0}\Vert(W^{k+1}-\widetilde{W}^{k},H^{k+1}-\widetilde{H}^{k}) \Vert^2\\
        &\leq J(W^{k},H^{k})-J(W^{k+1},H^{k+1}).
    \end{split}
\end{equation}
% Let us first prove the following equation
% \begin{equation}
%     \begin{split}
%         &\rho_{1}\Vert(W^{k+1}-\widetilde{W}^{k},H^{k+1}-\widetilde{H}^{k}) \Vert^2\\
%         &\leq J(W^{k},H^{k})-J(W^{k+1},H^{k+1})
%     \end{split}
% \end{equation}
% If the update of $\boldsymbol{W}^k$ and $\boldsymbol{H}^k$ are obtained by (10), then the condition assumed in (11) is true, that is, the above formula is true;
% If the update of $\boldsymbol{W}^k$ and $\boldsymbol{H}^k$ are obtained by (12-13), then we know $\widetilde{\boldsymbol{H}}^{k} = \boldsymbol{H}^{k}, \widetilde{\boldsymbol{W}}^{k} = \boldsymbol{W}^{k}$, combined with the equation (\ref{eq:25}), we can see that the above formula still holds.
% Moreover, given Assumption 1, $J$ is bounded from below, leading to the convergence of the sequence $J(W^{k},H^{k})$.
The above Eq. (\ref{eq:28}) can be accumulated from $k = 0$ to $K$,
\begin{equation*}%加*表示不对公式编号
    \begin{split}
        &\sum_{k=0}^{K} \rho_{0}\Vert\left(W^{k+1}-\widetilde{W}^{k}, H^{k+1}-\widetilde{H}^{k}\right)\Vert^{2}\\
        &\leq J\left(W^{0}, H^{0}\right)-J\left(W^{K+1}, H^{K+1}\right)\\
        &\leq J\left(W^{0}, H^{0}\right)-J^{*}<\infty.
    \end{split}
\end{equation*}%加*表示不对公式编号
Let $K\rightarrow\infty$, we know
\begin{equation*}
    \lim_{k \to \infty}\rho_{0}\Vert\left(W^{k+1}-\widetilde{W}^{k}, H^{k+1}-\widetilde{H}^{k}\right)\Vert^{2}=0,
\end{equation*}
hence
\begin{equation}
    \lim_{k\to\infty} \Vert W^{k+1}-\widetilde{W}^{k} \Vert =0, 
    \lim_{k\to\infty} \Vert H^{k+1}-\widetilde{H}^{k} \Vert =0.
\end{equation}
This means the Lemma \ref{lem-1} is true.
\end{proof}

\begin{lem}\label{lem-2}
    \text{(A subgradient lower bound for the iterates gap)} Suppose that Assumption 1 hold. Let $\{(\boldsymbol{W}^k,\boldsymbol{H}^k)\}$, $\{(\widetilde{\boldsymbol{W}}^k,\widetilde{\boldsymbol{H}}^k)\}$ be the sequences generated by Algorithm 2. For each integer $k\geq1$, define
    \begin{equation}
        \begin{split}
            p_{W}^{k+1} &= \nabla_{W}F(W^{k+1},H^{k+1}) - \nabla_{W}F(\widetilde{W}^{k},\widetilde{H}^{k}) \\
            &+ c_{k}(\widetilde{W}^{k}-W^{k+1}),
        \end{split}
    \end{equation}
    \begin{equation}
        \begin{split}
            p_{H}^{k+1} &= \nabla_{H}F(W^{k+1},H^{k+1}) - \nabla_{H}F(W^{k+1},\widetilde{H}^{k}) \\
            &+ d_{k}(\widetilde{H}^{k}-H^{k+1}),
        \end{split}
    \end{equation}
    then $q_{W}^{k+1} \in \partial_{W}J(W^{k+1},H^{k+1})$, $p_{H}^{k+1} \in \partial_{H}J(W^{k+1},H^{k+1})$ and there exists $\rho > 0$ such that
    \begin{equation}
        \Vert(p_{W}^{k+1},p_{H}^{k+1})\Vert\leq \rho\Vert (W^{k+1}-\widetilde{W}^{k},H^{k+1}-\widetilde{H}^{k})\Vert.
    \end{equation}
\end{lem}
\begin{proof}
    From the iterative scheme of $W$, we can get
    \begin{equation}
        \begin{split}
            W^{k+1}&\in\underset{W\in\mathbb{R}^{p\times r}}{\arg\min}\{\langle W-\widetilde{W}^{k},\nabla_{W}F(\widetilde{W}^{k},\widetilde{H}^{k})\rangle\\
            &+\frac{c_{k}}{2}\Vert W-\widetilde{W}^{k}\Vert+f(W)\},
        \end{split}
    \end{equation}
    By optimality condition yields, we have
    \begin{equation}
        \nabla_{W}F(\widetilde{W}^{k},\widetilde{H}^{k})+c_{k}(W^{k+1}-\widetilde{W}^{k})+u^{k+1}=0,
    \end{equation}
    where $u^{k+1}\in \partial f(W^{k+1})$. It follows that
    \begin{equation}
        \begin{split}
            p_{W}^{k+1} &= \nabla_{W}F(W^{k+1},H^{k+1})-\nabla_{W}F(\widetilde{W}^{k},\widetilde{H}^{k})\\
            &+c_{k}(\widetilde{W}^{k}-W^{k+1})\\
            &=\nabla_{W}F(W^{k+1},H^{k+1})+u^{k+1}\in\partial_{W}J(W^{k+1},H^{k+1}).
        \end{split}
    \end{equation}
    In a similar way, we know
    \begin{equation}
        \begin{split}
            p_{H}^{k+1} &= \nabla_{H}F(W^{k+1},H^{k+1})-\nabla_{H}F(W^{k+1},\widetilde{H}^{k})\\
            &+d_{k}(\widetilde{H}^{k}-H^{k+1})\\
            &=\nabla_{H}F(W^{k+1},H^{k+1})+v^{k+1}\in\partial_{H}J(W^{k+1},H^{k+1}),
        \end{split}
    \end{equation}
    where $v^{k+1}\in \partial g(H^{k+1})$.
    There exists $M>0$ such that
     \begin{equation*}%加*表示不对公式编号
        \begin{split}
            \Vert p_{W}^{k+1}\Vert&\leq c_{k}\Vert \widetilde{W}^{k}-W^{k+1}\Vert+\Vert \nabla_{W}F(W^{k+1},H^{k+1})\\
            &-\nabla_{W}F(\widetilde{W}^{k},\widetilde{H}^{k})\Vert\\
            &\leq c_{k}\Vert \widetilde{W}^{k}-W^{k+1}\Vert+M(\Vert W^{k+1}-\widetilde{W}^{k}\Vert\\
            &+\Vert H^{k+1}-\widetilde{H}^{k}\Vert)\\
            &=(M+c_k)\Vert \widetilde{W}^{k}-W^{k+1}\Vert+M\Vert H^{k+1}-\widetilde{H}^{k}\Vert\\
            &\leq (2M+c_{k})\Vert (W^{k+1}-\widetilde{W}^{k},H^{k+1}-\widetilde{H}^{k})\Vert.
        \end{split}
    \end{equation*}
    By the same way, because $\nabla_{H}F$ is Lipschitz continuous, we have
    \begin{equation*}
        \begin{split}
            \Vert p_{H}^{k+1}\Vert&\leq d_{k}\Vert\widetilde{H}^{k}-H^{k+1}\Vert+\Vert\nabla_{H}F(W^{k+1},H^{k+1})\\
            &- \nabla_{H}F(W^{k+1},\widetilde{H}^{k})\Vert\\
            &\leq d_{k}\Vert \widetilde{H}^{k}-H^{k+1}\Vert+d_{k}\Vert H^{k+1}-\widetilde{H}^{k}\Vert\\
            &=2d_k\Vert H^{k+1}-\widetilde{H}^{k}\Vert,
        \end{split}
    \end{equation*}
    hence
    \begin{equation}
        \begin{split}
            \Vert (p_{W}^{k+1},p_{H}^{k+1})\Vert&\leq \Vert p_{W}^{k+1}\Vert+\Vert p_{H}^{k+1}\Vert\\
            &\leq\rho\Vert (W^{k+1}-\widetilde{W}^{k},H^{k+1}-\widetilde{H}^{k})\Vert.
        \end{split}
    \end{equation}
    This means the Lemma \ref{lem-2} is true.
\end{proof}

In the following outcomes, we summarize several properties of the set of limit points. For convenience, we often use the notation $z^{k}:=\left(W^{k}, H^{k}\right), \forall k\geq0$ in the following text. We first give the definition of the limit points set, let $\{z^k\}$, where $k\in \mathbb{N}$, be the sequence generated by Algorithm \ref{alg-3} from the initial point $z^0$. The collection of all limit points is denoted by $\omega(z^0)$, which is defined as follows:
\begin{equation}
    \begin{split}
        &\omega\left(z^{0}\right)=\{\bar{z}=(\bar{W}, \bar{H}) \in \mathbb{R}^{p\times r} \times \mathbb{R}^{r\times n}:\\
        &\exists\text{ a sequence of integers } \left\{k_{j}\right\}_{j \in \mathbb{N}}\text{such that }z^{k_{j}} \rightarrow \bar{z}, j \rightarrow \infty\}.  
    \end{split}
\end{equation}
\begin{lem}\label{lem-3}
    (Subsequence convergence) Let Assumption \ref{assumption-1} holds, $\{z_k\}$ is the sequence generated by Algorithm \ref{alg-3}, the following assertion hold.\\
    (i) $\emptyset \neq \omega(z^{0})\subset\text{crit}J$.\\
    (ii) $\omega(z^{0})$ is a nonempty and compact set, and the objective function $J$ is finite and constant on $\omega(z^{0})$.\\
    (iii) We have
    \begin{equation*}
        \lim_{k\to\infty}dist(z^k,\omega(z^{0}))=0.
    \end{equation*}
\end{lem}
\begin{proof}
    (i) From Lemma \ref{lem-1}, we know
    \begin{equation}
        % \begin{split}
            \lim _{j \rightarrow \infty}\left\Vert W^{k_{j}}-\widetilde{W}^{k_{j}-1}\right\Vert=0, 
            \lim _{j \rightarrow \infty}\left\Vert H^{k_{j}}-\widetilde{H}^{k_{j}-1}\right\Vert=0,
        % \end{split}
    \end{equation}
    hence
    \begin{equation}
        \begin{split}
            \lim _{j \rightarrow \infty} W^{k_{j}}=\lim _{j \rightarrow \infty} \widetilde{W}^{k_{j}-1}=\bar{W},\\
            \quad \lim _{j \rightarrow \infty} H^{k_{j}}=\lim _{j \rightarrow \infty} \widetilde{H}^{k_{j}-1}=\bar{H}.
        \end{split}
    \end{equation}
    From Lemma \ref{lem-2}, we have
    \begin{equation}
        \Vert (p_{W}^{k_j},p_{H}^{k_j})\Vert\leq \rho\Vert (W^{k_j}-\widetilde{W}^{k_j-1},H^{k_j}-\widetilde{H}^{k_j-1})\Vert,
    \end{equation}
    further we can conclude
    \begin{equation}
        (p_{W}^{k_j},p_{H}^{k_j})\rightarrow (0,0) \text{ as } j\rightarrow \infty.
    \end{equation}
    According to $\lim_{j \rightarrow \infty}j(W^{k_{j}}, H^{k_{j}})=j(\bar{W}, \bar{H})$, $(q_{W}^{k_j}, p_{H}^{k_j})\in \partial J(W^{k_{j}}, H^{k_{j}})$ and $\partial J$ closed, we get $(0, 0)\in \partial J(\bar{W}, \bar{H})$, this proves that the $\bar{z}=(\bar{W},\bar{H})$ is a critical point of $J$, that is to say $\omega(z^{0})\subset\text{crit}J$.\\
    (ii) The boundedness of the sequence $\{z^k\}$ implies that the set $\omega(z^0)$ is not empty. Furthermore, the expression for $\omega(z^0)$ can be redefined as the intersection of compact sets
    \begin{equation}
        \omega(z^{0}) = \bigcap_{s \in \mathbb{N}} \overline {\bigcup_{k\geq s}\{z_{k}\}},
    \end{equation}
    demonstrating that $\omega(z^0)$ forms a compact set and for any $\bar{z}=(\bar{W},\bar{H})\in \omega(z^{0})$, there exist a subsequence $\{z^{k_j}\}$ such that
    \begin{equation}
        \lim _{j \rightarrow \infty} z^{k_{j}}=\bar{z}.
    \end{equation}
    Since $f$ and $g$ are lower semicontinuous, we obtain that 
    \begin{equation}\label{eq-41}
        \lim_{j\to \infty}\inf f(W^{k_j}) \geq f(\bar{W}), \lim_{j\to \infty}\inf g(H^{k_j}) \geq g(\bar{H}).
    \end{equation}
    From the iterative step in Algorithm \ref{alg-3}, we know that for all integer k
    \begin{equation}
    \begin{split}
        W^{k+1} \in \operatorname{argmin}_{W} &\{\langle W-\widetilde{W}^k,\nabla_{W}F(\widetilde{W}^{k},H^{k})\rangle+\\
        &\frac{c_k}{2}\Vert W-\widetilde{W}^k\Vert^2+f(W)\},
    \end{split}
    \end{equation}
    letting $W=\bar{W}$, we can get
    \begin{equation}
    \begin{split}
        \langle W^{k+1}-\widetilde{W}^k&,\nabla_{W}F(\widetilde{W}^{k},H^{k})\rangle+\frac{c_k}{2}\Vert W^{k+1}-\widetilde{W}^k\Vert^2+\\
        f(W^{k+1})&\leq\langle\bar{W}-\widetilde{W}^k,\nabla_{W}F(\widetilde{W}^{k},H^{k})\rangle+\\
        &\frac{c_k}{2}\Vert \bar{W}-\widetilde{W}^k\Vert^2+f(\bar{W}).
    \end{split}
    \end{equation}
    Then letting $k=k_{j}-1$, 
    \begin{equation}\label{eq-44}
    \begin{split}
        &\lim_{j\to \infty}\sup f(W^{k_j})\\
        &\leq\lim_{j\to \infty}\sup(\langle \bar{W}-\widetilde{W}^{k_j-1},\nabla_{W}F(\widetilde{W}^{k_j-1},H^{k_j-1})\rangle\\
        &+\frac{c_k}{2}\Vert\bar{W}-\widetilde{W}^{k_j-1}\Vert^2)+f(\bar{W}),
    \end{split}
    \end{equation}
    since $\widetilde{W} ^{k_j-1}\to \bar{W}$ as $j\to \infty$, Eq.(\ref{eq-44}) reduces to
    \begin{equation}
        \lim_{j\to \infty}\sup f\left(W^{k_j}\right)\leq f\left(\bar{W}\right),
    \end{equation}
    thus, by combining Eq.(\ref{eq-41}), we obtain
    \begin{equation}
        \lim_{j\to \infty}f\left(W^{k_j}\right)=f\left(\bar{W}\right).
    \end{equation}
    Applying a similar argument to $g$ and $H^k$, we get
    \begin{equation}
        \lim_{j\to \infty}g\left(H^{k_j}\right)=g\left(\bar{H}\right).
    \end{equation}
    Considering the continuity of $F$ (see Assumption \ref{assumption-1}), we can finally obtain
    \begin{equation}
        \lim_{j\to \infty}J\left(z^{k_j}\right)=J\left(\bar{z}\right),
    \end{equation}
    which means that $J$ is constant on $\omega(z^{0})$.\\
    % From Lemma \ref{lem-1}, we know that $J\left(z^{k_{j}}\right)$ converges to a constant such that
    % \begin{equation}
    %     \lim _{j \rightarrow \infty} J\left(z^{k_{j}}\right)=l,
    % \end{equation}
    % and then the sequence $\{J(z^{k_j})\}_{j \in \mathbb{N}}$ converges to $l$. Hence $J(\bar{z})=l$, which means that $J$ is constant on $\omega(z^{0})$.\\  
    % Since $J$ is continuous, we have
    % \begin{equation}
    %     \lim _{j \rightarrow \infty} J\left(z^{k_{j}}\right)=J(\bar{z}),
    % \end{equation}
    % and because ${J(z^k)}$ converges globally to a critical point $J^{*}$ (See Lemma \ref{lem-1}), then
    % \begin{equation}
    %     \lim _{j \rightarrow \infty} J\left(z^{k_{j}}\right)=\lim _{k \rightarrow \infty} J\left(z^{k}\right)=J^{*}.
    % \end{equation}
    % 把“=J(\bar{z})”删了
    % This means that $J$ is a constant in $\omega(z^{0})$.\\
    (iii)
    This term is a basic result of the limit point definition, or we can use the strategy of disproving. Suppose $\lim_{k\to\infty}dist(z^k,\omega(z^{0}))\neq 0$, then there is a subsequence $\{z^{k_m}\}$ and a constant $M$ greater than zero, so that
    \begin{equation}
        \Vert z^{k_m}-\bar{z}\Vert \geq dist(z^{k_m},\omega(z^{0}))>M\text{, }\forall\bar{z}\in \omega(z^{0}).
    \end{equation}
    On the other hand, $\{z^{k_m}\}$ is bounded and has a subsequence $\{z^{k_{m_j}}\}$ converges to a point in $\omega(z^{0})$, so
    \begin{equation}
        \lim_{k\to\infty}dist(z^k,\omega(z^{0}))=0.
    \end{equation}
This means the Lemma \ref{lem-3} is true.
\end{proof}

\begin{theorem}\label{theorem-1}
    (Global convergence)
    Assuming that $J$ is a KL function and $\{z^k\}$ is a sequence generated by algorithm \ref{alg-3} with initial point $z^0$. Suppose Assumptions \ref{assumption-1} holds then it has the following properties:\\
    (i) $\sum_{k=0}^{\infty}\left\|z^{k+1}-z^k\right\|<\infty$.\\
    (ii) The sequence $\{z^k\}$ converges to the critical point of $J$.
\end{theorem}
\begin{proof}
    (i) Since $\{J(z^k)\}$ is a non-increasing sequence, and from Lemma \ref{lem-3} we know that $lim_{k \to \infty}J(z^k)=J(\bar{z})$, it is clear that $J(\bar{z})<J(z^k)$ for all $k>0$ and for any $\eta>0$, there exists a positive integer $k_0$ such that $J(z^{k_0})<J(\bar{z})+\eta$, hence for all $k>k_0$ we have
    \begin{equation}\label{eq-49}
        z^k\in \left[ J(\bar{z})<J(z^{k})<J(\bar{z})+\eta\right].
    \end{equation}
    From Lemma \ref{lem-3} (iv) we know that $lim_{k \to \infty}dist(z^k,\omega(z^{0}))=0$, thus for any $\epsilon>0$ there exists a positive integer $k_1$ such that
    \begin{equation}\label{eq-50}
        dist(z^k,\omega(z^{0}))<\epsilon, \forall k>k_1.
    \end{equation}
    Based on Eq. (\ref{eq-49}) and (\ref{eq-50}), it is observed that by letting $l=\max\{k_0,k_1\}$, for any $k>l$, we have
    \begin{equation}
        z^k\in \{ z|\text{dist}(z^k,\omega(z^{0}))<\epsilon \}\cap \left[J(\bar{z})<J(z^{k_0})<J(\bar{z})+\eta \right].
    \end{equation}
    Since $J$ is finite and  constant on the nonempty and compact set $\omega(z^{0})$ (see Lemma \ref{lem-3}), according to Lemma 6 in \cite{bo2014proximal}, there exists a concave function $\varphi$, for any $k>l$ we get that
    \begin{equation}\label{eq-52}
        \varphi'\left(J(z^{k})-J(\bar{z})\right)\text{dist}(0,\partial J(z^k))\geq 1.
    \end{equation}
    From Lemma \ref{lem-2}, we have
    \begin{equation}\label{eq-53}
        \begin{split}
            \text{dist}(0,\partial J(z^k))
            &\leq \rho\left\Vert (W^k-\widetilde{W}^{k-1},H^k-\widetilde{H}^{k-1})\right\Vert\\
            &=\rho\left\Vert z^k-\omega^{k-1}\right\Vert,
        \end{split}
    \end{equation}
    where $\omega^{k-1}:=(\widetilde{W}^{k-1},\widetilde{H}^{k-1})$, combine Eq. (\ref{eq-52}) and (\ref{eq-53}), we get that
    \begin{equation}\label{eq-54}
        \varphi'\left(J(z^{k})-J(\bar{z})\right)\geq \frac{1}{\rho\left\Vert z^k-\omega^{k-1}\right\Vert}.
    \end{equation}
    From the convexity of $\varphi$ we get
    \begin{equation}
        \begin{split}
            &\varphi\left(J(z^{k+1})-J(\bar{z})\right)-\varphi\left(J(z^{k})-J(\bar{z})\right) \\
            &\leq\varphi'\left(J(z^{k})-J(\bar{z})\right)\left( J(z^{k+1})-J(z^k)\right),
        \end{split}
    \end{equation}
    and from Lemma \ref{lem-1} and Eq. (\ref{eq-54}) we have
    \begin{equation}\label{eq-56}
        \begin{split}
            &\varphi\left(J(z^{k})-J(\bar{z})\right)-\varphi\left(J(z^{k+1})-J(\bar{z})\right)\\
            &\geq \varphi'\left(J(z^{k})-J(\bar{z})\right)\left( J(z^{k})-J(z^{k+1})\right)\\
            &\geq \frac{\rho_0\left\Vert z^{k+1}-\omega^{k}\right\Vert^2}{\rho\left\Vert z^{k}-\omega^{k-1}\right\Vert}.
        \end{split}
    \end{equation}
    For convenience, we define $\varDelta_k=\varphi\left(J(z^k)-J(\bar{z})\right)$, let $C=\frac{\rho}{\rho_0}$, then Eq. (\ref{eq-56}) can be simplified as
    \begin{equation}
        \varDelta_k-\varDelta_{k+1}\geq \frac{\left\Vert z^{k+1}-\omega^{k}\right\Vert^2}{C\left\Vert z^{k}-\omega^{k-1}\right\Vert}.
    \end{equation}
    It follows that
    \begin{equation}
        \left\Vert z^{k+1}-\omega^{k}\right\Vert^2\leq C(\varDelta_k-\varDelta_{k+1})\left\Vert z^{k}-\omega^{k-1}\right\Vert,
    \end{equation}
    \begin{equation}\label{eq-59}
        2\left\Vert z^{k+1}-\omega^{k}\right\Vert\leq C(\varDelta_k-\varDelta_{k+1})+\left\Vert z^{k}-\omega^{k-1}\right\Vert.
    \end{equation}
    Summing up Eq. (\ref{eq-59}) for $k=l+1,...,K$ yields
    \begin{equation}
        \begin{split}
            &2\sum_{k=l+1}^K\left\Vert z^{k+1}-\omega^{k}\right\Vert\\
            &\leq\sum_{k=l+1}^K\left\Vert z^{k}-\omega^{k-1}\right\Vert+ C(\varDelta_{l+1}-\varDelta_{K+1})\\
            &\leq\left\Vert z^{l+1}-\omega^{l}\right\Vert+\sum_{k=l+1}^K\left\Vert z^{k+1}-\omega^{k}\right\Vert+C(\varDelta_{l+1}-\varDelta_{K+1}),
        \end{split}
    \end{equation}
    and hence, let $K\to \infty$, we get
    \begin{equation}
        \sum_{k=l+1}^\infty\left\Vert z^{k+1}-\omega^{k}\right\Vert\leq\left\Vert z^{l+1}-\omega^{l}\right\Vert+C\varDelta_{l+1} < \infty.
    \end{equation}
    This shows that the sequence $\{z^k\}$ is finite, and
    \begin{equation}
        \sum_{k=0}^\infty\left\Vert z^{k+1}-\omega^{k}\right\Vert<\infty.
    \end{equation}
    Study the iteration point $\omega^k=(\widetilde{W}^k,\widetilde{H}^k)$ in Algorithm \ref{alg-3}. We note that if $(\widetilde{W}^k,\widetilde{H}^k)$ is generated by (5)(8), then
     \begin{equation}\label{eq-63}
        \begin{split}
            &(W^{k+1}-\widetilde{W}^k,H^{k+1}-\widetilde{H}^k)\\
            &=(W^{k+1}-W^{k}-\beta_k(W^{k}-W^{k-1})\\
            &,H^{k+1}-H^{k}-\beta_k(H^{k}-H^{k-1}))\\
            &=z^{k+1}-z^k-\beta_k(z^k-z^{k-1}).
        \end{split}
    \end{equation}
    If $(\widetilde{W}^k,\widetilde{H}^k)$ is generated by (13), then
    \begin{equation}\label{eq-64}
        \begin{split}
            (W^{k+1}-\widetilde{W}^k,H^{k+1}-\widetilde{H}^k)&=(W^{k+1}-W^k,H^{k+1}-H^k)\\
            &=z^{k+1}-z^k.
        \end{split}
    \end{equation}
    Combining Eq. (\ref{eq-63}) and (\ref{eq-64}) we can get
    \begin{equation}\label{eq-65}
        \left\Vert z^{k+1}-\omega^{k}\right\Vert\geq \left\Vert z^{k+1}-z^{k}\right\Vert-\beta_k\left\Vert z^{k}-z^{k-1}\right\Vert,
    \end{equation}
    and summing up Eq. (\ref{eq-65}) for $k=l+1,...,K$ yields
    \begin{equation}\label{eq-66}
        \begin{split}
            &\sum_{k=l+1}^K\left\Vert z^{k+1}-z^{k}\right\Vert-\sum_{k=l+1}^K\beta_k\left\Vert z^{k}-z^{k-1}\right\Vert\\
            &\leq \sum_{k=l+1}^K\left\Vert z^{k+1}-\omega^{k}\right\Vert<\infty,
        \end{split}
    \end{equation}
    % Note that $\beta_k\in[0,1]$, let $\bar{\beta}=\sup_k\{\beta_k\}$, then $0\leq\bar{\beta}<1$, the following formula holds
    and hence
    \begin{equation}
        \begin{split}\label{eq-67}
            &\sum_{k=l+1}^K\left\Vert z^{k+1}-z^{k}\right\Vert-\sum_{k=l+1}^K\beta_k\left\Vert z^{k}-z^{k-1}\right\Vert\\
            &\geq\sum_{k=l+1}^K\left\Vert z^{k+1}-z^{k}\right\Vert-\beta_{max}\sum_{k=l+1}^K\left\Vert z^{k+1}-z^{k}\right\Vert-\beta_k\left\Vert z^{l+1}-z^{l}\right\Vert\\
            &=(1-\beta_{max})\sum_{k=l+1}^K\left\Vert z^{k+1}-z^{k}\right\Vert-\beta_k\left\Vert z^{l+1}-z^{l}\right\Vert.
        \end{split}
    \end{equation}
    Combining Eq. (\ref{eq-66}) and (\ref{eq-67}) we have
    %  \begin{equation*}
    %     \begin{split}
    %         &\sum_{k=l+1}^K\left\Vert z^{k+1}-z^{k}\right\Vert-\bar{\beta}\sum_{k=l+1}^K\left\Vert z^{k}-z^{k-1}\right\Vert\\
    %         &=\sum_{k=l+1}^K\left\Vert z^{k+1}-z^{k}\right\Vert-\bar{\beta}\sum_{k=l+1}^K\left\Vert z^{k+1}-z^{k}\right\Vert-\bar{\beta}\left\Vert z^{l+1}-z^{l}\right\Vert\\
    %         &=(1-\bar{\beta})\sum_{k=l+1}^K\left\Vert z^{k+1}-z^{k}\right\Vert-\bar{\beta}\left\Vert z^{l+1}-z^{l}\right\Vert.
    %     \end{split}
    % \end{equation*}
    % Combine the above three formulas to get
    \begin{equation*}
        \begin{split}
            &(1-\beta_{max})\sum_{k=l+1}^K\left\Vert z^{k+1}-z^{k}\right\Vert<\infty.
        \end{split}
    \end{equation*}
    Let $K\to \infty$, from $0\leq\beta_{max}<1$ we can get
    \begin{equation}
        \sum_{k=l+1}^{\infty}\left\Vert z^{k+1}-z^{k}\right\Vert<\infty,
    \end{equation}
    which means
    \begin{equation}\label{eq-69}
        \sum_{k=0}^{\infty}\left\Vert z^{k+1}-z^{k}\right\Vert<\infty.
    \end{equation}
    (ii) Eq. (\ref{eq-69}) shows that with $m>n\geq l$, we have
    \begin{equation}\label{eq-70}
        \begin{split}
            \left\Vert z^m-z^n\right\Vert=\left\Vert \sum_{k=n}^{m-1}(z^{k+1}-z^k)\right\Vert&\leq\sum_{k=n}^{m-1}\left\Vert z^{k+1}-z^k\right\Vert\\
            &<\sum_{k=n}^{\infty}\left\Vert z^{k+1}-z^k\right\Vert.
        \end{split}
    \end{equation}
    From Eq. (\ref{eq-70}), we get $\left\Vert z^m-z^n\right\Vert$ converges to zero as $n\to\infty$, it follows that $\{z^k\}$ is a Cauchy sequence and thus a convergent sequence. From Lemma \ref{lem-3} (iii), it can be concluded that $\{z^k\}$ converges to a critical point of $J$.
\end{proof}

\begin{theorem}\label{theorem-2}
(Convergence rate)
Let Assumption \ref{assumption-1} hold and the desingularizing function has the form of $\varphi(t)=\frac{c}{\theta} t^{\theta}$ with $\theta \in \left(0,1\right]$, $c>0$. Let $J(z) = J^{*}$  for all  $z \in \omega(z^0)$, and denote  $r_{k}:=J\left(z^{k}\right)-J^{*}$. Then the sequence $\left\{r_{k}\right\}$ satisfies for $k_{2}$ large enough:
\begin{enumerate}
	\item {If $\theta=1$, then $r_{k}$ reduces to zero in finite steps;}
	\item {If $\theta\in\left[\frac{1}{2}, 1\right)$, then there exist a constant $P \in \left[0,1\right)$ such that $r_{k}\leq r_{k_{2}}P^{k-k_{2}}$;}
	\item {If $\theta\in\left(0,\frac{1}{2}\right)$, then there exist a constant $Q>0$ such that $r_{k}\leq\left(\frac{Q}{\left(k-k_{2}\right)(1-2\theta)}\right)^{\frac{1}{1-2 \theta}}$.}
\end{enumerate}
\end{theorem}
\begin{proof}
% The result is almost the same as it mentioned in \cite{li2017convergence}.
Checking the assumptions of Theorem 3 in reference \cite{li2017convergence}, we observe that all assumptions required in Algorithm \ref{alg-3} are clearly satisfied, so Theorem \ref{theorem-2} holds.
\end{proof}

\section{Experiments}
We evaluate the effectiveness of the proposed GNMF\_$\ell_{20}$ method for the clustering task and compare it with state-of-the-art matrix factorization methods, as well as the k-means and sparse k-means methods. The competing methods considered in the evaluation are as follows:
\begin{itemize}
    \item NMF: Non-negative Matrix Factorization.
    \item NMF\_$\ell_0$: NMF with $\ell_{0}$ sparsity constraint.
    \item SNMF\_$\ell_{20}$: Sparse NMF with $\ell_{20}$ sparsity constraint \cite{min2022structured}.
    \item SNMF\_$\ell_{c0}$: Sparse NMF with $\ell_{c0}$ sparsity constraint \cite{min2022structured}.
    \item GNMF: Graph NMF.
    \item GNMF\_$\ell_{0}$: Graph NMF with $\ell_{0}$ sparsity constraint.
    \item Kmeans: K-means clustering algorithm.
    \item Kmeans\_$\ell_{0}$: K-means clustering with $\ell_{0}$ sparsity constraint \cite{chang2018sparse}.
\end{itemize}
\subsection{Evaluation metrics}
In the experiment, we use Normalized Mutual Information (NMI) and Accuracy (ACC) as key metrics to evaluate clustering quality and accuracy.

Normalized Mutual Information (NMI) \cite{li2013clustering}:
NMI is a metric used to measure the similarity between the clustering results and t he ground truth labels. NMI ranges from 0 to 1, where a value closer to 1 indicates a higher consistency between the clustering results and the ground truth labels.

The formula to calculate NMI is as follows:
\begin{equation}
    \text{NMI} = \frac{2 \times I(C, T)}{H(C) + H(T)},
\end{equation}
where $C$ represents the clustering results, $T$ represents the ground truth labels, $I(C, T)$ denotes the mutual information between $C$ and $T$, and $H(C)$ and $H(T)$ represent the entropy of $C$ and $T$, respectively. 
% The calculation of mutual information involves the distributions of the clustering results and the ground truth labels, while the calculation of entropy involves the probability distributions of the clustering results and the ground truth labels.

Accuracy is a metric used to measure the correctness of the clustering results. It represents the ratio of correctly classified samples to the total number of samples. Accuracy ranges from 0 to 1, where a value closer to 1 indicates a higher accuracy of the clustering results.
The formula to calculate accuracy is as follows:
\begin{equation}
    \text{Accuracy} = \frac{\sum_{i=1}^{n} h(y_i, \hat{y_i})}{n},
\end{equation}
where $n$ represents the total number of samples, $y_i$ represents the true labels of the $i$-th sample, $\hat{y_i}$ represents the clustering result of the i-th sample, and $h(y_i, \hat{y_i})$ represents an indicator function that is 1 when $y_i$ and $\hat{y_i}$ are equal and 0 otherwise.
% By computing NMI and accuracy, we can evaluate the performance of the clustering algorithm on the given dataset. These metrics provide quantitative assessments of the consistency, accuracy, and reliability of the clustering results. Additionally, we can further explain and discuss the significance and implications of these metrics in conjunction with other experimental results and analyses.
\subsection{Application to synthetic data}
\begin{figure*}[h]
  \centering \includegraphics[width=1\linewidth]{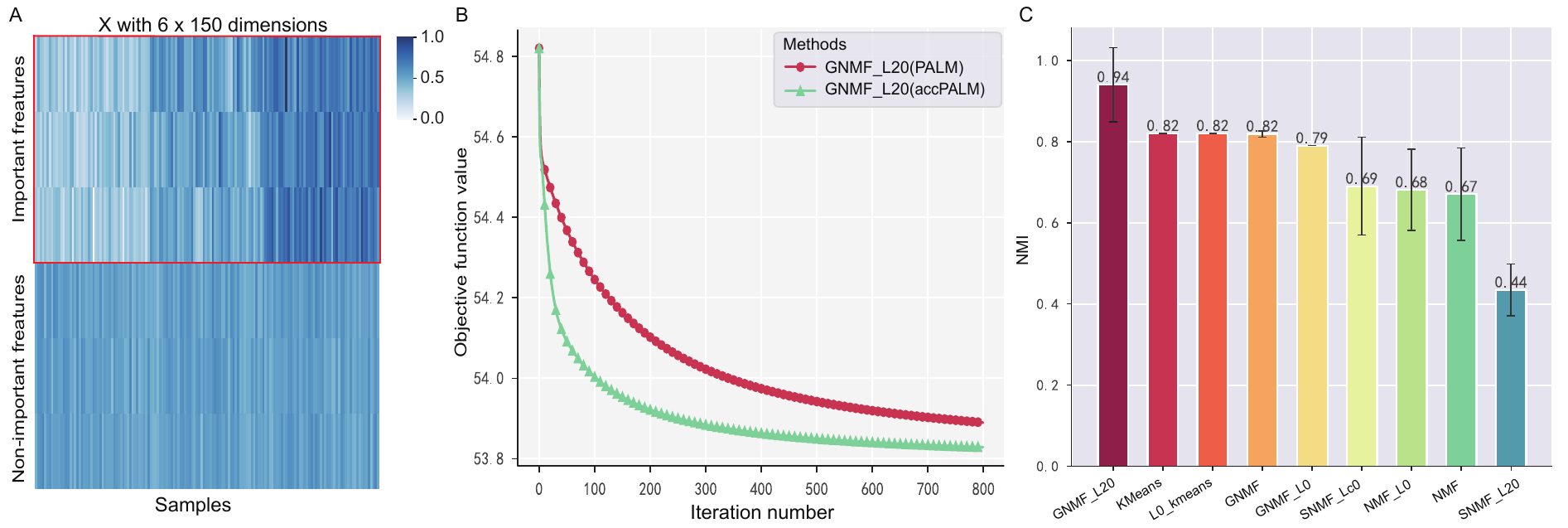}
  \caption{
  Results on the synthetic data.
  (A) Heatmap showing the synthetic data.
  (B) Convergence performance of PALM and accPALM for GNMF\_$\ell_{20}$ and the initial value of the inertial parameter $\beta$ of the acceleration method is 0.5.
  (C) Comparison of nine unsupervised clustering methods in terms of NMI on the synthetic data.
  }\label{fig-2}
\end{figure*}
We constructs a synthetic dataset with three Gaussian clusters. Each cluster is generated by sampling from a multivariate Gaussian distribution with a specified mean and covariance matrix. The means of the three clusters are set as -2, 0, and 2, respectively, with a covariance matrix of identity.
To introduce variability and noise into the dataset, we add random noise by generating samples from a normal distribution and appending them to the original data. The noise is added only to the last three rows of the dataset, while the rest of the data remains unchanged.
% To ensure consistency and comparability across the dataset, a linear normalization process is performed. This involves subtracting the minimum value from each data point and dividing the result by the range of values in the dataset. This normalization procedure transforms the original data onto a new scale ranging from 0 to 1. By applying this normalization technique, all data points are brought within a standardized range, enabling meaningful comparisons and facilitating subsequent analysis in the experiment.
For consistency and comparability across the dataset, a linear normalization is performed. This involves subtracting the minimum value from each data point and dividing the result by the range of values, transforming the data to a new scale from 0 to 1.
Finally, we perturb the order of the last three rows by shuffling their columns randomly. And the synthetic data set we constructed is shown in Figure \ref{fig-2}A.
% By constructing synthetic data in this manner, we create a controlled environment to evaluate the performance of our proposed methods and competing algorithms in clustering tasks.

In order to conduct experiments on the GNMF\_$\ell_{20}$ method and achieve better results, we manually constructed the adjacency matrix for graph regularization. By customizing the adjacency matrix, we can incorporate prior knowledge or desired characteristics into the graph regularization process, thereby guiding the clustering algorithm to capture relevant information and improve clustering performance. Then we introduce the construction method of adjacency matrix as follows.

Firstly, we initialized a square matrix with zeros, representing the absence of connections between data points. Then, we divided the matrix into blocks corresponding to different clusters or groups in the dataset. Within each block, we randomly assigned non-zero values to represent the presence of connections between data points. The probabilities of assigning non-zero values were set according to a predefined distribution, reflecting the desired level of connectivity within each cluster. Finally, we ensured symmetry by setting the corresponding elements in the lower triangular part of the matrix equal to the corresponding elements in the upper triangular part.

\begin{table*}[ht]
\centering
\caption{Averaged performance of GNMF\_$\ell_{20}$ with different algorithms in terms of $Relative~Error = \frac{\|\bm{X} - \bm{W}\bm{H}\|_F}{\|\bm{X}\|_F}$ and running time (seconds) on the LIBRAS, UMIST and USPS datasets. The best results are highlighted in bold.
}\label{tab-3}
\begin{tabular}{r|cc|cc|cc}
  \hline
  Data & \multicolumn{2}{c}{LIBRAS data} & \multicolumn{2}{c}{UMIST data} & \multicolumn{2}{c}{JAFFE data} \\
   \cmidrule(lr){2-3} \cmidrule(lr){4-5}\cmidrule(lr){6-7}
  Metrics & Relative Error & Time (seconds) & Relative Error & Time (seconds) & Relative Error & Time (seconds) \\
  \hline
  BPL     & \textbf{0.736} & 47.007          & \textbf{0.971} & 133.238          & 0.564          & 52.882 \\
  PALM    & 0.745          & 48.961          & 0.990          & 75.352           & 0.566          & 51.134 \\
  iPALM   & 0.746          & 47.758          & 0.991          & 78.736           & 0.604          & 60.445 \\
  accPALM & 0.741          & \textbf{40.252} & \textbf{0.971} & \textbf{20.004}  & \textbf{0.563} & \textbf{47.746} \\
  
  \hline
\end{tabular}
\end{table*}

We demonstrate the convergence performance of PALM and accPALM in solving the GNMF\_$\ell_{20}$ problem on synthetic dataset in Figure \ref{fig-2}B. The results indicate that accPALM exhibits significantly faster convergence compared to the standalone PALM method. Subsequently, we compare the GNMF\_$\ell_{20}$ method with other classical clustering algorithms, revealing that our approach outperforms the other classical methods on synthetic dataset (See Figure \ref{fig-2}C). Our method excels in filtering out irrelevant features from the data matrix and the incorporation of graph structure further enhances its performance. Furthermore, we observed that the standard GNMF and SNMF$\_\ell_{20}$ methods are less effective compared to GNMF$\_\ell_{20}$. This underscores the performance improvement achieved through the combination of graph regularization and $\ell_{2,0}$-norm constraints.

\subsection{Application to real image datasets}
This study evaluates the proposed method and competing methods on three real-world image datasets, as follows:
\begin{itemize}
    \item \textbf{LIBRAS} dataset is a widely used dataset in sign language recognition research \cite{dua2017uci}. It consists of video sequences of Brazilian Sign Language (LIBRAS) gestures performed by different individuals. Each gesture is recorded from multiple viewpoints, capturing variations in hand shapes, movements, and orientations.
    \item \textbf{UMIST} dataset is a popular face recognition dataset. It contains grayscale images of faces from 20 different individuals, with varying poses, expressions, and lighting conditions.
    \item \textbf{JAFFE} dataset is a dataset for facial expression analysis. It consists of grayscale images of Japanese female faces displaying seven different facial expressions: neutral, happiness, sadness, surprise, anger, disgust, and fear.
\end{itemize}

We demonstrate the convergence performance when solving the GNMF\_$\ell_{20}$ model in Eq.(\ref{eq:11}) using PALM and accPALM on image data (Figure \ref{fig-3}). When the extrapolation parameter $\beta$ is 0, accPALM degenerates into the ordinary PALM algorithm. Across the three image datasets, various settings of extrapolation parameters exhibit similar impacts on convergence speed, and it is found that the convergence speed of our proposed accelerated method is significantly faster than that of the ordinary PALM method. And as the extrapolation parameter $\beta$ of the acceleration method increases, the convergence speed of the objective function also becomes faster, and the convergence speed is fastest when the parameter $\beta$ changes adaptively during the iteration process, so unless otherwise specified, the article The accPALM method adopts the strategy of $\beta$ adaptive change.
\begin{figure*}[h]
  \centering \includegraphics[width=1\linewidth]{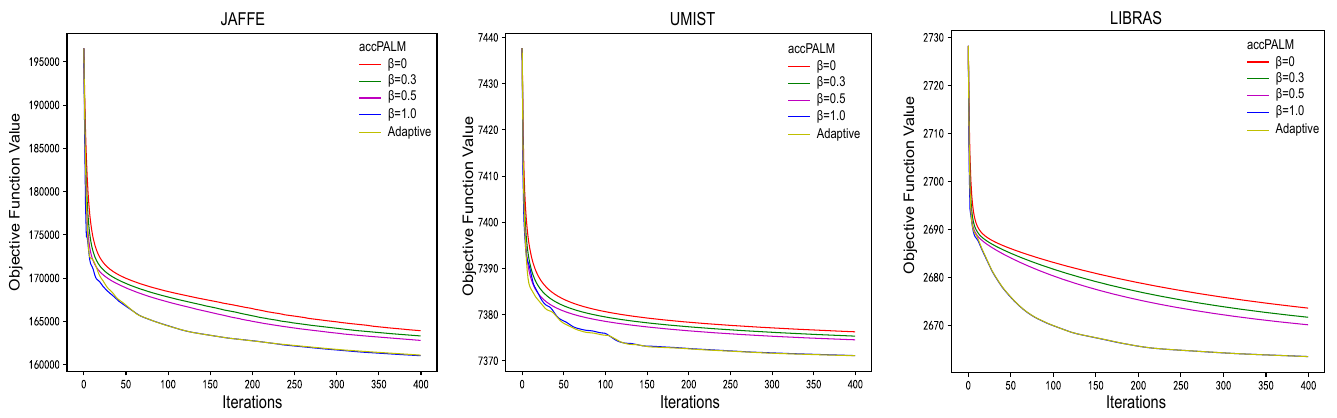}
  \caption{Convergence performance of accPALM with different parameters $\beta$ for GNMF\_$\ell_{20}$, where the first four lines represent the convergence performance when the extrapolation parameter $\beta$ is fixed, and the last line represents the convergence performance when the extrapolation parameter $\beta$ changes dynamically as the iteration proceeds. }\label{fig-3}
\end{figure*}

\begin{figure*}[h]
  \centering \includegraphics[width=1\linewidth]{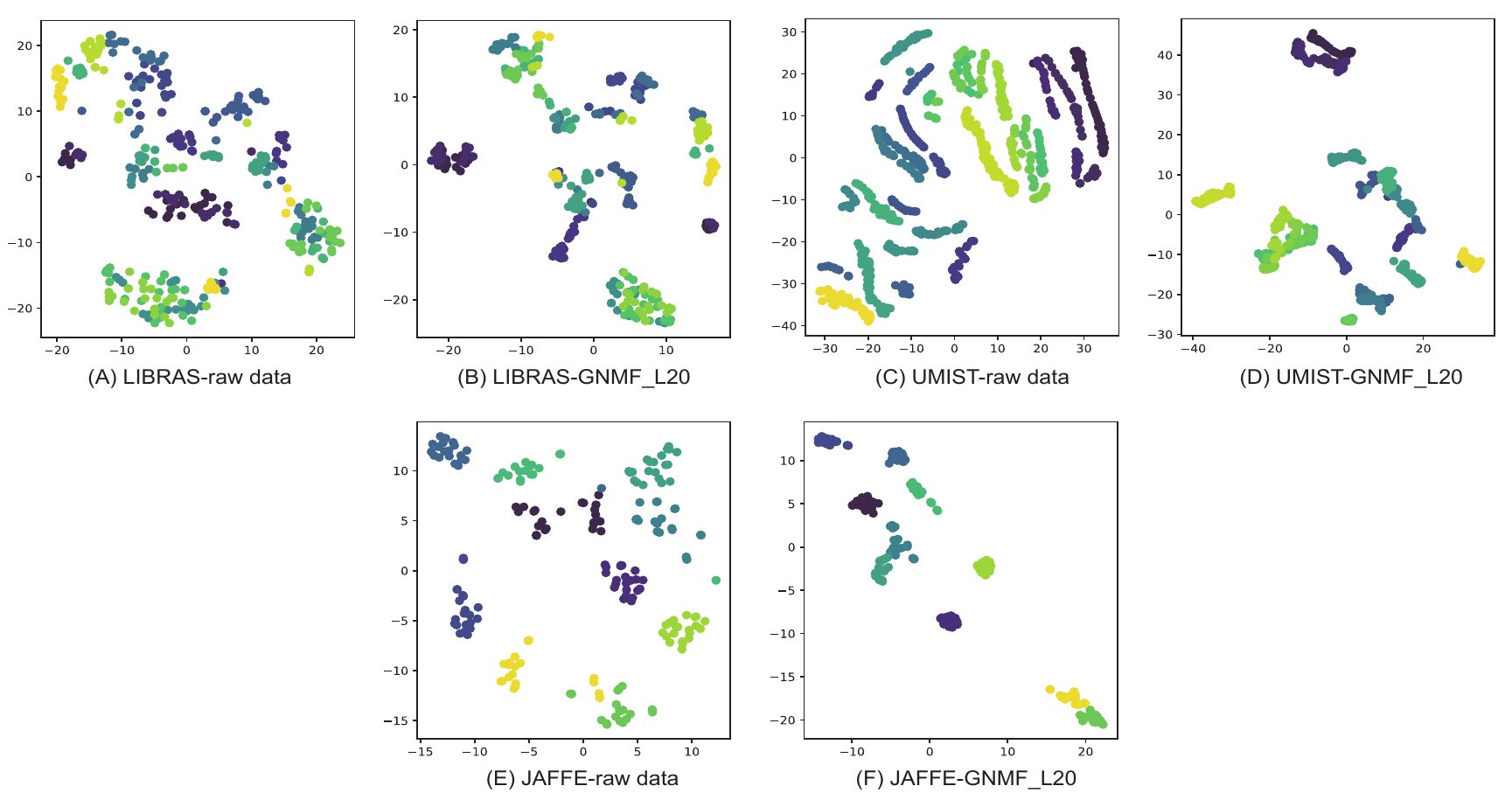}
   \caption{2D visualization results via t-SNE. The comparison of the raw data and the clustering results of GNMF\_$\ell_{20}$ on LIBRAS, UMIST and JAFFE datasets.}\label{fig-4}
\end{figure*}
% 加速算法与其他算法的比较！
We compare our proposed method with other commonly used optimization algorithms, including iPALM \cite{pock2016inertial} and BPL \cite{xu2017globally}, on three image datasets (Table \ref{tab-3}). To ensure fairness, all these algorithms were initialized with the same random starting points. We conducted ten repetitions of each algorithm, with each repetition using a different randomly generated initial point. Additionally, we set the termination criterion for all algorithms as $\frac{\|(\bm{W}^{t},\bm{H}^{t})-(\bm{W}^{t-1},\bm{H}^{t-1}) \|}{\| (\bm{W}^{t-1},\bm{H}^{t-1}) \|} \leq 10^{-3}$. We compare the convergence speed and relative error of these methods, and the results demonstrated that accPALM exhibited the fastest convergence rate and excellent performance in terms of relative error. Specifically, in UMIST dataset, the accPALM algorithm demonstrates significantly reduced computation time compared to other algorithms, enhancing convergence speed while maintaining good relative error.

We evaluate the clustering performance of these methods from two aspects, NMI (Normalized Mutual Information) and ACC (Accuracy) (Table \ref{tab-1}). We performed k-means clustering on the $\bm{H}$ matrix resulting from matrix factorization to obtain the final clustering results. The clustering evaluations are presented in Table \ref{tab-1}, where it is evident that our proposed GNMF\_$\ell_{20}$ method outperforms the others. In detail, the incorporation of graph regularization has a positive impact on clustering. The clustering index of GNMF is approximately $5\%$ to $15\%$ higher compared to the conventional NMF method. However, we observed that the enhancement of GNMF over NMF in the UMIST dataset is not pronounced. This could be attributed to the graph regularization module not effectively capturing the internal structure information of the data. The incorporation of $\ell_{2,0}$-norm constraint, however, may capture crucial features, thus improving the clustering performance.

\begin{table}[htbp]
\centering
\caption{Comparison in terms of (NMI \% $\pm$ std) and (ACC \% $\pm$ std) on the LIBRAS, UMIST and JAFFE datasets.}
\begin{adjustbox}{width=1\columnwidth,center}\label{tab-1}
\begin{tabular}{l|l|l|l}
  \hline
  \textbf{LIBRAS data}    &\#Features & NMI$\pm$sd &ACC$\pm$sd \\
  \hline
  GNMF\_$\ell_{20}$       & 40          & 65.66 $\pm$ 1.45 & 50.42 $\pm$ 2.59 \\
  NMF                     & all         & 59.45 $\pm$ 1.61 & 48.03 $\pm$ 2.20 \\
  NMF\_$\ell_{0}$         &\verb|/|     & 59.07 $\pm$ 0.92 & 48.36 $\pm$ 1.39 \\
  NMF\_$\ell_{20}$        & 40          & 58.08 $\pm$ 1.92 & 45.61 $\pm$ 1.96 \\
  NMF\_$\ell_{c0}$        &\verb|/|     & 59.42 $\pm$ 1.26 & 47.69 $\pm$ 1.87 \\
  GNMF                    & all         & 65.28 $\pm$ 1.30 & 49.28 $\pm$ 1.62 \\
  GNMF\_$\ell_{0}$        &\verb|/|     & 58.95 $\pm$ 1.62 & 49.00 $\pm$ 1.62 \\
  Kmeans                  & all         & 59.43 $\pm$ 0.71 & 44.42 $\pm$ 1.62 \\
  Kmeans$\_\ell_{0}$      & 40          & 40.73 $\pm$ 0.46 & 29.50 $\pm$ 0.73 \\
  \hline
  \hline
  \textbf{UMIST data}     &\#Features & NMI$\pm$sd &ACC$\pm$sd \\
  \hline
  GNMF\_$\ell_{20}$       & 50          & 74.76 $\pm$ 2.48 & 54.37 $\pm$ 3.07 \\
  NMF                     & all         & 61.13 $\pm$ 1.24 & 42.50 $\pm$ 1.73 \\
  NMF\_$\ell_{0}$         &\verb|/|     & 61.97 $\pm$ 2.00 & 39.97 $\pm$ 1.55 \\
  NMF\_$\ell_{20}$        & 2200        & 59.48 $\pm$ 2.49 & 42.10 $\pm$ 1.85 \\
  NMF\_$\ell_{c0}$        &\verb|/|     & 61.32 $\pm$ 1.07 & 40.09 $\pm$ 1.63 \\
  GNMF                    & all         & 64.56 $\pm$ 2.07 & 43.10 $\pm$ 3.08 \\
  GNMF\_$\ell_{0}$        &\verb|/|     & 74.09 $\pm$ 0.86 & 54.42 $\pm$ 1.67 \\
  Kmeans                  & all         & 65.34 $\pm$ 1.45 & 43.63 $\pm$ 1.83 \\
  Kmeans$\_\ell_{0}$      & 50          & 50.66 $\pm$ 1.38 & 36.70 $\pm$ 0.97 \\
  \hline
  \hline
  \textbf{JAFFE data}      &\#Features & NMI$\pm$sd &ACC$\pm$sd \\
  \hline
  GNMF\_$\ell_{20}$       & 50000       & 94.24 $\pm$ 2.73 & 94.98 $\pm$ 3.06 \\
  NMF                     & all         & 87.87 $\pm$ 4.05 & 86.29 $\pm$ 6.40 \\
  NMF\_$\ell_{0}$         &\verb|/|     & 88.11 $\pm$ 3.03 & 83.80 $\pm$ 4.79 \\
  NMF\_$\ell_{20}$        & 50000       & 86.24 $\pm$ 2.11 & 84.84 $\pm$ 2.59 \\
  NMF\_$\ell_{c0}$        &\verb|/|     & 87.48 $\pm$ 3.22 & 83.38 $\pm$ 4.61 \\
  GNMF                    & all         & 92.26 $\pm$ 4.30 & 90.89 $\pm$ 6.27 \\
  GNMF\_$\ell_{0}$        &\verb|/|     & 14.31 $\pm$ 0.39 & 22.82 $\pm$ 0.23 \\
  Kmeans                  & all         & 90.25 $\pm$ 3.48 & 85.68 $\pm$ 4.84 \\
  Kmeans$\_\ell_{0}$      & 50000       & 46.26 $\pm$ 0.58 & 33.99 $\pm$ 1.15 \\
  \hline
\end{tabular}
\end{adjustbox}
\end{table}
We also investigate the impact of the number of features on the algorithms' performance(Table \ref{tab-2}). By selecting various numbers of features and observing the clustering evaluations in our experiments, we found that GNMF\_$\ell_{20}$ consistently outperformed other methods across different feature counts. In the LIBRAS data set, we found that the clustering performance of GNMF\_$\ell_{20}$ generally gets better as the number of selected features increases. The clustering effectiveness peaks when the number of features reaches 50, suggesting that our method has identified the 50 most important features in the data, successfully mitigating noise interference.
\begin{table}[ht]
	\centering
	\caption{Comparison of (NMI \%) averages with different number of characteristic number ($k$) on the LIBRAS, UMIST and JAFFE datasets.}\label{tab-2}
    \resizebox{0.49\textwidth}{!}{
	\begin{tabular}{lllllll}
		\hline
		\textbf{LIBRAS data} ($k$)& 5 & 10 & 30 & 50 & 70 & 90 \\
		\hline
  
        GNMF\_$\ell_{20}$       & 48.63 & 59.58 & 66.53 & 66.30 & 65.34 & 64.50 \\
		NMF\_$\ell_{20}$        & 23.89 & 44.89 & 58.33 & 59.61 & 59.82 & 59.04 \\
		Kmeans$\_\ell_{0}$      & 32.13 & 36.40 & 46.39 & 51.86 & 57.28 & 59.00 \\
		\hline
		\hline
		\textbf{UMIST data} ($k$)& 200 & 500 & 1000 & 1500 & 2000 & 2500 \\
		\hline
  
        GNMF\_$\ell_{20}$       & 72.35 & 69.73 & 67.19 & 65.61 & 64.86 & 65.63 \\
		NMF\_$\ell_{20}$        & 49.46 & 56.88 & 59.11 & 60.47 & 58.69 & 59.27 \\
		Kmeans$\_\ell_{0}$      & 61.18 & 65.92 & 65.05 & 65.02 & 65.79 & 65.28 \\
		\hline
		\hline
	    \textbf{JAFFE data} ($k$)& 5000 & 10000 & 30000 & 50000 & 60000 & 65000 \\
		\hline
  
        GNMF\_$\ell_{20}$       & 93.19 & 93.40 & 92.82 & 94.18 & 92.58 & 90.83 \\
		NMF\_$\ell_{20}$        & 74.87 & 83.85 & 88.11 & 91.77 & 86.70 & 87.03 \\
		Kmeans$\_\ell_{0}$      & 60.85 & 66.93 & 79.16 & 86.40 & 90.29 & 88.43 \\
		\hline
	\end{tabular}
    }
\end{table}

We further use the t-SNE algorithmto visualize the clustering results of the three data sets in 2-dimensional space \cite{van2008visualizing}, as shown in Figure \ref{fig-4}. From Figure \ref{fig-4}, we can know that compared with the original data, the data points are driven into different groups after GNMF\_$\ell_{20}$ clustering, and points from the same cluster become more concentrated.
\section{Conclusion}
In this paper, we present a GNMF model with $\ell_{2,0}$-norm constraint and integrated feature selection and graph regularization into the NMF framework. We have known that $\ell_{2,0}$-norm satisfies the KŁ property, allowing us to utilize the PALM algorithm to solve the nonconvex and non-smooth optimization problems associated with $\ell_{2,0}$-norm constraint. Specifically, we have introduced the GNMF\_$\ell_{20}$ model and developed an accelerated version of PALM (accPALM) to solve it efficiently. 
Through extensive experiments on synthetic and real image datasets, we compare our proposed GNMF\_$\ell_{20}$ methods with competing methods for the clustering task. The results demonstrate the superior performance of our methods in terms of clustering accuracy and feature selection. 
Overall, our work contributes to the advancement of SSNMF models and their applications in clustering and feature selection tasks. The proposed algorithms have been supported by convergence theory and have demonstrated state-of-the-art performance. However, further research is needed to explore automatic parameter selection and global optimization solutions.
In conclusion, this study provides valuable insights into the GNMF models with $\ell_{2,0}$-norm constraint, presents efficient algorithms for solving them, and demonstrates their effectiveness in clustering and feature selection tasks. Our work contributes to the field of matrix factorization and offers potential for further advancements in data analysis and feature learning.

\balance
\small{
\bibliographystyle{IEEEtran}
\bibliography{MYREF}
}

% 最新修订 2024-01-14
\begin{IEEEbiography}[{\includegraphics[width=1in,height=1.25in,clip,keepaspectratio]{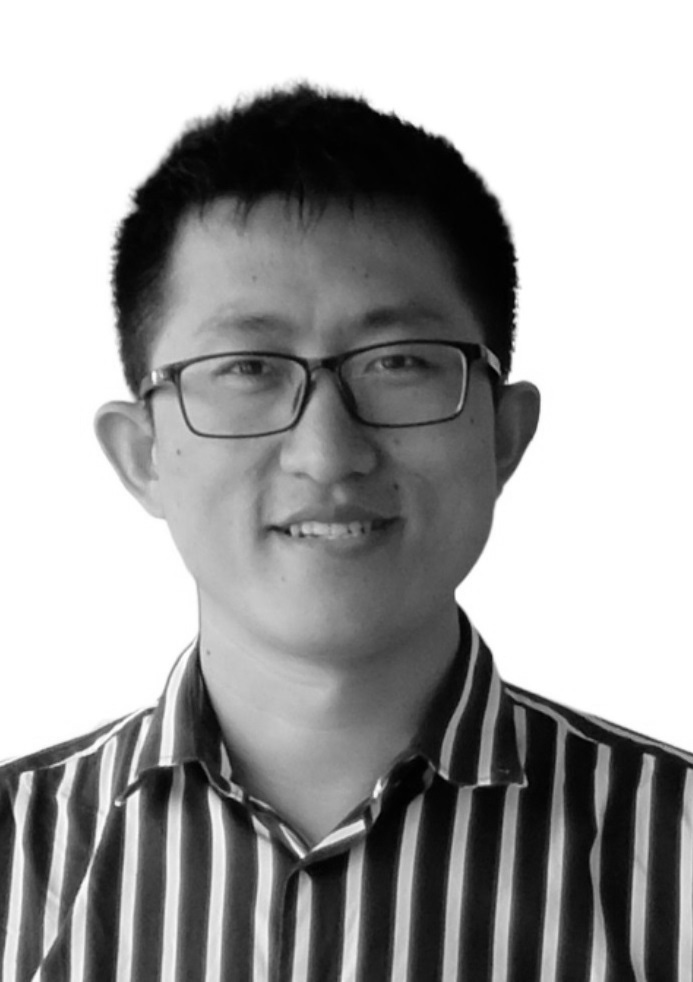}}]{Wenwen Min}
is currently an associate professor in the School of Information Science and Engineering, Yunnan University.
He received the Ph.D. degree in Computer Science from the School of Computer Science, Wuhan University, in 2017.
He was a visiting Ph.D student at the Academy of Mathematics and Systems Science, Chinese Academy of Sciences from 2015 to 2017.
He was a Postdoctoral Researcher with the School of Science and Engineering, The Chinese University of Hong Kong, Shenzhen, China, from 2019 to 2021.

His current research interests include machine/deep learning, sparse optimization and bioinformatics.
He has authored about 30 papers in journals and conferences, such as the
IEEE Trans Knowl Data Eng,
IEEE Trans Image Process,
IEEE Trans Neural Netw Learn Syst,
PLoS Comput Biol,
Bioinformatics, IEEE/ACM Trans Comput Biol Bioinform, and IEEE International Conference on Bioinformatics and Biomedicine.
\end{IEEEbiography}
\begin{IEEEbiography}[{\includegraphics[width=1in,height=1.25in,clip,keepaspectratio]{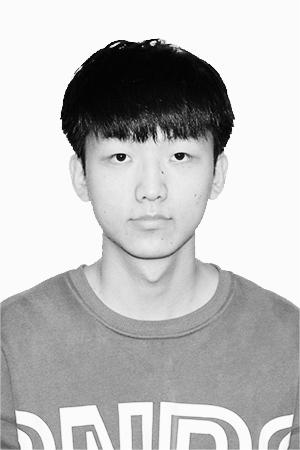}}]{Zhen Wang}
is currently pursuing a master's degree at the School of Information Science and Engineering, Yunnan University.

His current research interests include machine/deep learning, sparse optimization and graph machine learning.
\end{IEEEbiography}

\end{document}